%% file: main.tex
\documentclass[a4paper,onecolumn,11pt,accepted=2024-08-26]{quantumarticle}
\pdfoutput=1

\usepackage[numbers]{natbib}

\usepackage[utf8]{inputenc} 
\usepackage[T1]{fontenc}    
\usepackage[english]{babel}
\usepackage{hyperref}       
\usepackage{url}            
\usepackage{booktabs}       
\usepackage{amsfonts}       
\usepackage{nicefrac}       
\usepackage{microtype}      
\usepackage{xcolor}         

\usepackage{amssymb,amsmath,amsthm,amsfonts}
\usepackage{textgreek}
\usepackage{mathtools}
\usepackage{enumitem}
\usepackage{graphicx}
\usepackage{epstopdf}
\usepackage{subfigure}
\usepackage{float}
\usepackage[ruled,vlined,linesnumbered]{algorithm2e}
\usepackage{physics}
\usepackage{footnote}
\usepackage{xcolor}
\usepackage{mathrsfs}
\usepackage{bbm}
\usepackage{algorithmic}
\usepackage{enumitem}

\renewcommand*\citet[1]{\cite{#1}}
\renewcommand*\citep[1]{\cite{#1}}

\newtheorem{theorem}{Theorem}
\newtheorem{definition}{Definition}
\newtheorem{lemma}{Lemma}
\newtheorem{proposition}{Proposition}

\newtheorem{corollary}{Corollary}

\newcommand{\eqn}[1]{(\ref{eqn:#1})}
\newcommand{\eq}[1]{\eqref{eq:#1}}

\newcommand{\lem}[1]{Lemma~\ref{lem:#1}}

\newcommand{\rev}[1]{{\color{black}#1}}

\def\>{\rangle}
\def\<{\langle}

\newcommand{\ignore}[1]{}

\newcommand{\specialcell}[2][c]{%
  \begin{tabular}[#1]{@{}c@{}}#2\end{tabular}}

\newcommand{\C}{\mathbb{C}}

\def\Tr{\operatorname{Tr}}\def\:{\hbox{\bf:}}

\let\oldnl\nl
\newcommand{\nonl}{\renewcommand{\nl}{\let\nl\oldnl}}

\begin{document}

\title{Adaptive Online Learning of Quantum States}

\author{Xinyi Chen}
\affiliation{Department of Computer Science, Princeton University, NJ 08540, USA}
\affiliation{Google DeepMind Princeton, NJ 08542, USA}
\email{xinyic@princeton.edu}
\author{Elad Hazan}
\affiliation{Department of Computer Science, Princeton University, NJ 08540, USA}
\affiliation{Google DeepMind Princeton, NJ 08542, USA}
\email{ehazan@princeton.edu}
\author{Tongyang Li}
\affiliation{Center on Frontiers of Computing Studies, Peking University, 100871 Beijing, China}
\affiliation{School of Computer Science, Peking University, 100871 Beijing, China}
\email{tongyangli@pku.edu.cn}
\author{Zhou Lu}
\affiliation{Department of Computer Science, Princeton University, NJ 08540, USA}
\affiliation{Google DeepMind Princeton, NJ 08542, USA}
\email{zhoul@princeton.edu}
\author{Xinzhao Wang}
\affiliation{Center on Frontiers of Computing Studies, Peking University, 100871 Beijing, China}
\affiliation{School of Computer Science, Peking University, 100871 Beijing, China}
\email{wangxz@stu.pku.edu.cn}
\author{Rui Yang}
\affiliation{Center on Frontiers of Computing Studies, Peking University, 100871 Beijing, China}
\affiliation{School of Computer Science, Peking University, 100871 Beijing, China}
\email{ypyangrui@pku.edu.cn}

\maketitle

\begin{abstract}
The problem of efficient quantum state learning, also called shadow tomography, aims to comprehend an unknown $d$-dimensional quantum state through \rev{POVMs}. Yet, these states are rarely static; they evolve due to factors such as measurements, environmental noise, or inherent Hamiltonian state transitions. This paper leverages techniques from adaptive online learning to keep pace with such state changes.

The key metrics considered for learning in these mutable environments are enhanced notions of regret, specifically adaptive and dynamic regret. We present adaptive and dynamic regret bounds for online shadow tomography, which are polynomial in the number of qubits and sublinear in the number of measurements. To support our theoretical findings, we include numerical experiments that validate our proposed models. 
\end{abstract}

\section{Introduction}

As opposed to their classical counterparts, quantum states over $n$ qubits requires $2^n$ ``amplitudes'' for their description. Therefore, recovering the density matrix of a quantum system with $n$ qubits from measurements becomes an impractical task for even moderate values of $n$. The setting of shadow tomography~\citep{aaronson2018shadow} attempts to circumvent this issue: rather than recovering the entire state matrix, the goal is to accurately reconstruct the ``shadow'' that the density matrix casts on a set of measurements that are known in advance.

More precisely, let $\rho$ be an $n$-qubit quantum state: that is, a $2^n \times 2^n$ Hermitian positive semidefinite matrix with $\Tr (\rho) = 1$. We can measure $\rho$ by a two-outcome positive operator-valued measure (POVM) $E$, which can be represented by a $2^n\times 2^n$ Hermitian matrix with eigenvalues in $[0,1]$. Such a measurement $E$ accepts $\rho$ (returns 1) with probability $\Tr(E\rho)$ and rejects $\rho$ (returns 0) with probability $1-\Tr(E\rho)$. Given a sequence of POVMs $E_{1},\ldots,E_{m}$, shadow tomography uses copies of $\rho$ to estimate $\Tr(E_{i}\rho)$ within $\epsilon$ for all $i\in[m]$ with probability at least $2/3$. 

The state-of-the-art result~\citep{buadescu2021improved} can accomplish this task with $\tilde{O}(n(\log m)^{2} /\epsilon^{4})$ copies of $\rho$, exponentially improving upon full state tomography in terms of $n$. \rev{However, in the real world a POVM $E$ can be chosen based on previous actions from the observer (learner) and is not known in advance.}
Moreover, shadow tomography may be unfeasible on near-term quantum computers, as current methods~\citet{aaronson2018shadow,brandao2017exponential,aaronson2019gentle,buadescu2021improved} all require a joint measurement on multiple copies of $\rho$, which needs to be implemented on larger quantum systems limited by the scale of existing quantum devices.


In light of both concerns it is natural to consider \emph{online learning of quantum states}. In this problem, the measurements appear sequentially and the goal is to select a sequence of states to minimize regret compared to the optimal quantum state in retrospect. \rev{This problem naturally fits into the \emph{online convex optimization} framework, in which a player makes a decision at each step, and then is revealed a convex loss function and suffers the loss of its decision. The player makes decisions to minimize the total loss. However, this problem is too difficult since the loss functions can be arbitrary, so we only expect to minimize the loss of the player minus the minimum loss of making the same decision at every step, which is called the \emph{regret} of the player. In the online learning of quantum states setting, the player outputs a quantum state $\varphi_t$ at step $t$ and suffers a loss measuring the difference between $\varphi_t$ and the true state under the chosen POVM.} The sequential nature of the setting gives rise to algorithms that treat each measurement separately, thereby eliminating the need for joint measurements\footnote{\rev{In some shadow tomography algorithms using similar online learning schemes \citet{aaronson2018shadow,buadescu2021improved}, they do not need joint measurements when updating the state but do need them when choosing the POVMs}}. The existing algorithm of \cite{aaronson2018online}, along with subsequent works~\citep{yang2020revisiting,chen2020more,lumbreras2021multiarm, quantum_portfolio}, yields a sequence of states with regret $O(\sqrt{Tn})$ for $T$ iterations. \rev{In some cases, these regret bounds can be formalized as \emph{mistake bounds} which count the number of steps that the player suffers a loss larger than a given threshold.}


Nevertheless, existing literature on shadow tomography does not cover an important factor in current quantum computers: fluctuation of quantum states. In practical scenarios, we often cannot calibrate the quantum state constantly, and shadow tomography algorithms requiring multiple identical copies of $\rho$ may not be viable.
For instance, on the IBM quantum computing platform, a quantum circuit with measurements will be executed hundreds or thousands of shots in a row without calibrations in such consecutive experiments. 

In real-world experiments, a quantum state $\rho$ is typically obtained by evolving a control Hamiltonian $H$ for some time $\tau$ starting from an initial state $\rho_{0}$, i.e., $\rho=e^{iH\tau}\rho_{0}e^{-iH\tau}$. If the device is not calibrated for a while, the control Hamiltonian implemented is a noisy one $H+\delta H$, and the actual quantum state we prepare is a \rev{perturbed state $\rho'=e^{-i(H+\delta H)\tau}\rho_0 e^{i(H+\delta H)\tau}$}. 
More discussions about noises in quantum states can be found in~\citet{ball2016effect,greenbaum2017modeling,kueng2016comparing,wallman2015bounding,wallman2015estimating}.

In light of these fluctuations, recent study on state tomography extends to temporal quantum tomography~\citep{tran2021temporal} with changing quantum states. Therefore, a fundamental question from our perspective of machine learning is learning changing quantum states, which naturally fits into the framework of \emph{adaptive online learning}. Such quantum algorithms are crucial for near-term noisy, intermediate-scale quantum computers (NISQ), since the design of noise-resilient algorithms may extend the computational power of NISQ devices~\cite{Preskill2018NISQ}.

\paragraph{Adaptive online learning.} The motivation above suggests a learning problem that assumes the quantum state prepared at time step $t = 1,2,\ldots, T$ is $\rho_t$. Due to imperfect calibration, $\rho_t$ at different time $t$ is different, and our goal is to learn $\rho_t$ in an online fashion. The problem of online learning of a changing concept has received significant attention in the machine learning literature, and we adapt several of these techniques to the quantum setting, which involves high-dimensional spaces over the complex numbers.

The first metric we consider is the "dynamic regret" introduced by \citet{zinkevich2003online}, which measures the difference between the learner's loss and that of a changing comparator. The dynamic regret bounds are usually characterized by how much the optimal comparator changes over time, known as the \rev{"path length."} 

Next, we consider a more sophisticated metric, the \rev{"adaptive regret,"} as introduced by \citet{hazan2009efficient}. This metric measures the maximum of the regret over all intervals, essentially taking into account a changing comparator. Many extensions and generalizations of the original technique have been presented in works such as \citet{daniely2015strongly,jun2017improved,cutkosky2020parameter}. Minimizing adaptive regret requires a combination of expert algorithms and continuous online learning methods, which we also adopt for the quantum setting.

\paragraph{Contributions.} 
In this paper, we systematically study adaptive online learning of quantum states. In quantum tomography, we might have prior knowledge about $\rho_t$, but in online learning, it can be adversarial. As such, we expect regret guarantees that extend beyond just comparing with the best fixed $\varphi \in C_n$. We study the following questions summarized in Table~\ref{tab:main}:
\begin{itemize}[leftmargin=*]

\item \textbf{Dynamic regret:} We consider minimizing regret under the assumption that the comparator $\varphi$ changes slowly:
\begin{small}
\begin{align}
\text{Regret}_{T,\mathcal{P}}^{\mathcal{A}}=\sum_{t=1}^T \ell_{t}\left(\Tr(E_t x_{t}^{\mathcal{A}})\right)-\min_{\varphi_{t}\in C_n} \sum_{t=1}^T \ell_t\left(\Tr(E_t\varphi_{t})\right), \text{ where } \mathcal{P}=\sum_{t=1}^{T-1} \|\varphi_t-\varphi_{t+1}\|_\rev{1}. \label{eq:dynamic-regret}
\end{align}
\end{small}

\noindent
Here $\mathcal{P}$ is bounded and defined as the path length, the functions $\ell_{t}$ are $L$-Lipschitz for all $t\in[T]$, and $\|\cdot\|_\rev{1}$ is the $\ell_{1}$ norm of the singular values of a matrix known as the nuclear norm. We propose an algorithm based on online mirror descent (OMD) that achieves an $\tilde{O}(\sqrt{nT \mathcal{P}})$ regret bound. Although the analysis shares the same spirit as the standard OMD proof, tackling optimization in the complex domain requires new tools from several areas. Specifically, we use Klein's inequality from matrix analysis to study the dual update in OMD, and the Fannes–Audenaert inequality \cite{Audenaert_2007} from quantum information theory to obtain the dynamic regret bounds. Compared to \citet{aaronson2018online}, our analysis is more general due to the introduction of these tools. 

As an application, under the assumption that $\varphi_t$ evolves according to a family of dynamical models, we can also prove $\tilde{O}(\sqrt{nT \mathcal{P}'})$ dynamic regret bounds where $\mathcal{P}'$ is a variant of path length defined in Eq.~\eq{dynamic-regret-family}. This covers two natural scenarios: either the dynamics comes from a family of $M$ known models (which incurs a $\log M$ overhead), or we know that the dynamical model is an $O(1)$-local quantum channel, a natural assumption for current quantum computers fabricated by 2D chips~\citep{arute2019supremacy,IBM-Prague,zhu2022quantum} where a quantum gate can only influence a constant number of qubits.

\item \textbf{Adaptive regret:} We also consider minimizing the strongly adaptive regret \citep{daniely2015strongly} of an algorithm $\mathcal{A}$ at time $T$ for length $\tau$:
\begin{align}
    \text{SA-Regret}_T^{\mathcal{A}}(\tau)  = \max_{I\subseteq [T], |I| = \tau}\bigg(\sum_{t \in I} \ell_{t}\left(\Tr(E_t x_{t}^{\mathcal{A}})\right)
    -\min_{\varphi \in C_n} \sum_{t \in I} \ell_t\left(\Tr(E_t\varphi)\right)\bigg) 
    \label{eq:SA-regret}
\end{align}

where $x_t^{\mathcal{A}}$ is the output of algorithm $\mathcal{A}$ at time $t$, and $\ell_t$ is a real loss function that is revealed to the learner, for example, $\ell_t(z) = |z-b_t|$ and $\ell_t(z) = (z-b_t)^2$ for some $b_t\in [0,1]$.\footnote{Here we make no assumptions about how $\ell_t$ is revealed. In application, for example \cite{aaronson2018shadow}, they first obtain $b_t$ by repeated measurements and then compute $\ell_t$ classically.} We show that combining algorithms from \citet{aaronson2018online}, \citet{jun2017improved} gives an $O(\sqrt{n \tau \log(T)})$ bound for adaptive regret. We then derive an $O(\sqrt{knT\log(T)})$ regret bound when the comparator $\varphi$ is allowed to shift $k$ times. 

\item \textbf{Mistake bound:} Given a sequence of two-outcome measurements $E_1,E_2,\ldots, E_T$ where each $E_t$ is followed afterward by a value $b_t$ such that $|\Tr(E_t\rho_t)-b_t|\leq\epsilon/3$. Assuming the ground truth $\rho_t$ has $k$-shift or bounded path length $\mathcal{P}$, we build upon the regret bounds to derive upper bounds of the number of mistakes made by an algorithm whose output is a sequence of $x_t^{\mathcal{A}}\in [0,1]$, where $|x_t^{\mathcal{A}}-\Tr(E_t\rho_t)| > \epsilon$ is considered a mistake. For dynamic regret and adaptive regret settings, we give mistake bounds in Corollary~\ref{cor:path-length-mistake} and Corollary~\ref{cor:k-shift-mistake}, respectively.

\end{itemize}

\begin{table}[htbp]
\centering
\resizebox{\columnwidth}{!}{
\begin{tabular}{cccc}
\hline
Regret & Reference & Setting & O-Bound \\ \hline\hline
Dynamic & Theorem \ref{thm:dynamic} & Path length & $ L\sqrt{T(n+\log (T)) \mathcal{P}}$ \\ \hline
Dynamic & Corollary \ref{cor:dynamic-model-family} &  \specialcell{Family of $M$ dynamical models} & \specialcell{$L\sqrt{T(n+\log(T))\mathcal{P}'}$+$\sqrt{T\log(M\log(T))}$} \\ \hline
Dynamic & Corollary \ref{cor:dynamic-model-local-channel} & \specialcell{$l$-local quantum channels} & \specialcell{$L\sqrt{T(n+\log (T))\mathcal{P}'}$ $+2^{6l}\sqrt{T}$} \\ \hline
\hline
Adaptive & \rev{Lemma \ref{lem:SA-Regret}} & \specialcell{Arbitrary interval with length $\tau$} & $L \sqrt{n \tau \log (T)}$ \\ \hline
Adaptive & \rev{Theorem \ref{thm:k_shift regret}} & $k$-shift & $L\sqrt{knT \log (T)}$ \\ \hline
\end{tabular}
}
\caption{Adaptive and dynamic regret bounds of online learning of quantum states. Here $n$ is the number of qubits, $L$ is the Lipschitzness of loss functions, $T$ is the total number of iterations, $\mathcal{P}$ is the path length defined in Eq.~\eq{dynamic-regret}, $\mathcal{P}'$ in Corollary \ref{cor:dynamic-model-family} is a variant of path length defined in Eq.~\eq{dynamic-regret-family} and $\mathcal{P}'$ in Corollary \ref{cor:dynamic-model-local-channel} is defined in Eq.~\eq{l_local_path_length}. $l$-local quantum channels are explained in Section~\ref{sec:omd_dynamic_regret}.}
\vspace{-2mm}
\label{tab:main}
\end{table}

We also note that our results can be extended from two-outcome POVMs to $K$-outcome POVMs for $K\geq 2$, inspired by a recent work by~\citet{gong2022learning}. A $K$-outcome measurements $\mathcal{M}$ can be described by $K$ $2^n\times 2^n$ Hermitian matrices $E_1,\ldots,E_K$ such that $0\preceq E_j\preceq I,$ $\sum_{j=1}^K E_j = I$, and for any quantum state $\rho$, the measurement $\mathcal{M}$ outputs $j$ with probability $\mathrm{Tr}(E_j\rho)$. In our setting of online learning of quantum states, a sequence of measurements $\mathcal{M}_t$ is given and each $\mathcal{M}_t$ can be described by a set of Hermitian matrices $\{E_{t,1},E_{t,2},\ldots,E_{t,K}\}$. In Appendix~\ref{append:k-outcome}, we show that our algorithm has the same regret upper bound on it for any $K$ if we use the total variational norm as the loss function.

We also conduct numerical experiments for our algorithms about their dynamic and adaptive regret bounds on the $k$-shift and the path length settings. The results demonstrate that our algorithms improve upon non-adaptive algorithms when the quantum state is changing, confirming our theoretical results.

\section{Preliminaries}
In online convex optimization (refer to \cite{hazan2016introduction} for a comprehensive treatment), a player engages in a $T$-round game against an adversarial environment. In each round $t$, the player selects $x_t\in \mathcal{K}$, then the environment unveils the convex loss function $\ell_t$, and the player incurs the loss $\ell_t(x_t)$. The player's objective is to minimize the (static) regret:
\begin{align}
\text{Regret}=\sum_{t=1}^T \ell_t(x_t)-\min_{\varphi \in \mathcal{K}} \sum_{t=1}^T \ell_t(\varphi).
\end{align}
We can model the quantum tomography problem within the framework of online convex optimization: the convex domain $\mathcal{K}$ is taken to be $C_n$, the set of all trace-1 positive semidefinite (PSD) complex matrices of dimension $2^n$:
$$C_n=\{M\in \mathbb{C}^{2^n\times 2^n}, M=M^{\dagger}, M\succeq 0, \Tr(M)=1\}.$$
The loss function can be $\ell_t(x)=|\Tr(E_t x_t)-\Tr(E_t \rho_t)|$ or $\ell_t(x)=(\Tr(E_t x_t)-\Tr(E_t \rho_t))^2$ for example, measuring the quality of approximating the ground-truth $\rho_t$. 

Another goal is to bound the maximum number of mistakes we make during the learning process. We consider the case where $\forall t\in[T]$, $b_t$ in the revealed loss function $\ell_t$ satisfies $\ |b_t-\Tr(E_t\rho_t)| \le \frac{1}{3}\epsilon$. We consider it a mistake if the algorithm outputs an $x_t$ such that $|\Tr(E_t x_t)-\Tr(E_t\rho_t)| > \epsilon$.

\paragraph{Notations.} Let $\|\cdot\|_\rev{1}$ denote the nuclear norm, i.e., the $\ell_{1}$ norm of the singular values of a matrix. Let $x\bullet y = \Tr(x^\dagger y)$ denote the inner trace product between $x$ and $y$. Denote $\lambda_i(X)$ to be the $i^{\text{th}}$ largest eigenvalue of a Hermitian matrix $X$. $X\succeq 0$ means that the matrix $X$ is positive semidefinite. For any positive integer $n$, denote $[n]:=\{0,1,\ldots,n-1\}$. Denote $\otimes$ to be the tensor product, i.e., for matrices $A=(A_{i_{1}j_{1}})_{i_{1}\in[m_{1}],j_{1}\in[n_{1}]}\in\C^{m_{1}}\times\C^{n_{1}}$ and $B=(B_{i_{2}j_{2}})_{i_{2}\in[m_{2}],j_{2}\in[n_{2}]}\in\C^{m_{2}}\times\C^{n_{2}}$, we have $A\otimes B=(A_{i_{1}j_{1}}\cdot B_{i_{2}j_{2}})_{i\in[m_{1}m_{2}], j\in[n_{1}n_{2}]}\in\C^{m_{1}m_{2}}\times\C^{n_{1}n_{2}}$ where $i=i_{1}m_{2}+i_{2}$, $j=j_{1}n_{2}+j_{2}$. Throughout the paper, $\tilde{O}$ omits poly-logarithmic terms. 
 
\section{Minimizing Dynamic Regret}
\label{sec:omd_dynamic_regret}
In this section, we consider the case that $\rho_t$ may change over time. In particular, we do not assume the number of times that $\rho_t$ changes, but instead consider the total change over time measured by the path length. To this end, we study dynamic regret, where the path length $\mathcal{P}$ of the comparator in nuclear norm is restricted:
\begin{align}
\mathcal{P}=\sum_{t=1}^{T-1} \|\varphi_t-\varphi_{t+1}\|_\rev{1}.
\end{align}
Note that the set of comparators $\{\varphi_t\}$ do not have to be the ground truth states $\{\rho_t\}$; they can be any states satisfying the path length condition. We propose an algorithm achieving an $O(L\sqrt{T(n+\log (T)) \mathcal{P}})$ dynamic regret bound, which instantiates copies of OMD algorithms with different learning rates as experts, and uses the multiplicative weight algorithm to select the best expert. The need for an expert algorithm arises because to obtain the dynamic regret bound, the learning rate of the OMD algorithm has to be a function of the path length, which is unknown ahead of time. The main algorithm is given in Algorithm~\ref{algo:md+exp}, and the expert OMD algorithm is given in Algorithm~\ref{algo:algomd}. The following theorem gives the dynamic regret guarantee:

\begin{theorem}\label{thm:dynamic}
Assume the path length $\mathcal{P}=\sum_{t=1}^{T-1} \|\varphi_t-\varphi_{t+1}\|_\rev{1}\ge 1$ and the loss $\ell_{t}$ is convex, L-Lipschitz, and maps to $[0,1]$, the dynamic regret of Algorithm \ref{algo:md+exp} is bounded by $O(L \sqrt{T(n+\log (T)) \mathcal{P}})$.
\end{theorem}

\begin{algorithm}[h]
\caption{Dynamic Regret for Quantum Tomography} \label{algo:md+exp}
\begin{algorithmic}[1]
\STATE \textbf{Input:} a candidate set of $\eta$, $S=\{2^{-k-1}\mid 1\le k\le \log T\}$, constant $\alpha$.
\STATE Initialize $\log T$ experts $A_1,\ldots,A_{\log T}$, where $A_k$ is an instance of Algorithm~\ref{algo:algomd} with $\eta=2^{-k-1}$ for all $k\in[\log T]$.
\STATE Set initial weights $w_1(k)=\frac{1}{\log T}$ for all $k\in[\log T]$.
\FOR{$t = 1, \ldots, T$}
\STATE Predict $x_t=\frac{\sum_k w_t(k) x_t(k)}{\sum_k w_t(k)}$, where $x_t(k)$ is the output of the $k$-th expert.
\STATE Observe loss function $\ell_t(\cdot)$.
\STATE Update the weights as
$$
w_{t+1}(k)=w_t(k) e^{-\alpha \ell_t(x_t(k))}.
$$
\STATE Send gradients $\nabla \ell_t(x_t(k))$ to each expert $A_k$.
\ENDFOR
\end{algorithmic}
\end{algorithm}

\begin{algorithm}[h]
\caption{OMD for Quantum Tomography} \label{algo:algomd}
\begin{algorithmic}[1]
\STATE \textbf{Input:}  domain $\mathcal{K}=(1-\frac{1}{T})C_n +\frac{1}{T 2^n}I$, step size $\eta <\frac{1}{2L}$.
\STATE Define $R(x)=\sum_{k=1}^{2^n} \lambda_k(x) \log \lambda_k(x)$, $\nabla R(x) := I + \log (x)$, and let $B_R$ denote the Bregman divergence defined by $R$.
\STATE Set $x_1=2^{-n} I$, and $y_1$ to satisfy $\nabla R(y_1)=\bf{0}$.
\FOR{$t = 1, \ldots, T$}
\STATE Predict $x_t$ and receive loss $\ell_t(\Tr(E_t x_t))$.
\STATE Define $\nabla_t= \ell'_t (\Tr(E_t x_t)) E_t$, where $\ell_t'(y)$ is a subderivative of $\ell_t$ with respect to $y$.
\STATE Update $y_{t+1}$ such that $\nabla R(y_{t+1})=\nabla R(x_t)-\eta \nabla_t$.
\STATE Update $ x_{t+1} = \text{argmin}_{x \in \mathcal{K}} B_R(x||y_{t+1})$.
\ENDFOR
\end{algorithmic}
\end{algorithm}

The proof of Theorem~\ref{thm:dynamic} is deferred to Appendix~\ref{append:proof-thm-dynamic} and we discuss some technical challenges here. Ref.~\citet{zhang2018adaptive} gives an $O(\sqrt{T \mathcal{P}})$ dynamic regret bound for online convex optimization:
\begin{align}
\hspace{-1.5mm}\text{Dynamic Regret}=\sum_{t=1}^T \ell_t(x_t)-\min_{\varphi_1, \ldots, \varphi_t\in \mathcal{K}} \sum_{t=1}^T \ell_t(\varphi_t)
\end{align}
where $\sum_{t=1}^{T-1} \|\varphi_t-\varphi_{t+1}\|_2= \mathcal{P}$. Directly adopting their result will lead to an {\bf exponential dependence on $n$}, because the original algorithm uses $\ell_2$ norm to measure the path length, which makes dual norm of the gradient depends polynomially on the dimension ($2^n$ in this case). To avoid this issue, we use OMD with von Neumann entropy regularization to obtain a poly-logarithmic dependence on the dimension. From~\citet{aaronson2018online}, we know that with the von Neumann entropy as the regularizer, the regret can be bounded by problem parameters measured in nuclear norm and spectral norm, which tend to be small for online shadow tomography. The analysis is in a similar vein as the classical OMD proof~\cite{hazan2016introduction}, but classical results often rely on theorems from calculus that cannot be straightforwardly applied to the complex domain. 

Given the challenge of complex numbers, we use tools from functional analysis and quantum information theory to study the convergence of the algorithm. Our analysis deviates from the classical analysis in the following aspects:

\begin{itemize}[leftmargin=*]
\item We use the notion of Gateaux differentiability and the Gateaux derivative to extend classical gradient-based methods to the quantum domain. In particular, we use the Gateaux derivative to define the Bregman divergence, a key ingredient of the OMD analysis. Subsequently, we show that the Bregman divergence obeys an Euler's inequality that corresponds to the generalized Pythagorean theorem in the classical analysis \cite{hazan2016introduction}.

\item To give a quantitative bound on the Bregman divergence between the current iterate and the dual update, we cannot rely on the interpretation of the Bregman divergence as a notion of distance. Instead, we upper bound the quantity using the power series of the matrix exponential, and the Golden-Thompson inequality.

\item When bounding the dynamic regret, we need to bound the difference of von Neumann entropy between two quantum states. This can be seen as a type of continuity bound for the von Neumann entropy, and was first given by \citet{fannes} and later sharpened by \citet{Audenaert_2007}. The bound is known as the Fannes-Audenaert inequality with numerous applications in quantum information theory~\cite{wolf2008area,aaberg2014catalytic,winter2016operational,chitambar2019quantum}.
\end{itemize}

\paragraph{Mistake bound of the path length setting.}
We can obtain a mistake bound in this setting when we assume the ground truth $\rho_t$ has a bounded path length.

\begin{corollary}\label{cor:path-length-mistake}

Let $\{\rho_t\}$ be a sequence of $n$-qubit mixed state whose path length $P_T = \sum_{t=1}^{T-1}\|\rho_{t+1}-\rho_t\|_\rev{1} \le \mathcal{P}$, and let $E_1,E_2,\ldots$ be a sequence of two-outcome measurements that are revealed to the learner one by one, each followed by a value $b_t \in[0, 1]$ such that $|\Tr(E_t\rho_t)-b_t| \le \frac{1}{3}\epsilon$. Then there is an explicit strategy for outputting hypothesis states $x_1, x_2, \ldots$ such that $|\Tr(E_t x_t) - \Tr(E_t\rho_t)| > \epsilon$ for at most $\tilde{O}(\frac{L^2 n \mathcal{P}}{\epsilon^2})$ times.
\end{corollary}

\paragraph{Dynamic regret given dynamical models.}
So far, we have not made assumptions on how the state evolves. However, in many cases, the quantum state only evolves under a family of dynamical models. A natural question is the following: If we know the family of possible dynamical models in advance, can we achieve a better dynamic regret?

First, we consider the case when there is a fixed family of dynamical models. The most general linear maps $\Phi\colon C_{n}\to C_{n}$ implementable on quantum computers are \emph{quantum channels} satisfying:
\begin{itemize}[leftmargin=*]
\item Completely positive: For any $\rho\in C_{n}$ and $m$-dimensional identity matrix $I_{m}$, $\rho\otimes I_{m}\succeq 0$;
\item Trace preserving: $\Tr(\Phi(\rho))=\Tr(\rho)=1$ for any quantum state $\rho\in C_{n}$.
\end{itemize}

A common example is $\Phi(x)=e^{iHt}xe^{-iHt}$, where $H$ is a background Hamiltonian that evolves the quantum state following the Schr{\"o}dinger equation.

We first consider the case with $M$ possible dynamical models $\{\Phi^{(1)},\ldots,\Phi^{(M)}\}$, each of which is a quantum channel. Following \citet{hall2013dynamical} and \citet{zhang2018adaptive}, we restrict a variant of the path length of the comparator:
\begin{align}\label{eq:dynamic-regret-family}
\mathcal{P}' = \min_{\Phi^{(i)}\in\{\Phi^{(1)},\ldots,\Phi^{(M)}\}}\sum_{t = 1}^{T-1} \|\varphi_{t+1}-\Phi^{(i)}(\varphi_{t})\|_\rev{1}.
\end{align}

\begin{corollary}\label{cor:dynamic-model-family}
Under the assumption that we have $M$ possible dynamical models $\{\Phi^{(1)},\ldots,\Phi^{(M)}\}$ and the path length $\mathcal{P}' = \min_{\Phi^{(i)}\in\{\Phi^{(1)},\ldots,\Phi^{(M)}\}}\sum_{t = 1}^{T-1} \|\varphi_{t+1}-\Phi^{(i)}(\varphi_{t})\|_\rev{1}\ge1$, there is an algorithm with dynamic regret $O(L\sqrt{T(n+\log(T))\mathcal{P}'}+\sqrt{T\log(M\log(T))})$.
\end{corollary}

In general, the possible dynamical models can be a continuous set. Considering that current quantum computers~\citep{arute2019supremacy,IBM-Prague,zhu2022quantum} commonly use 2D chips where a quantum gate only influences a constant number of qubits, it is natural to consider $l$-local quantum channels where each channel only influences at most $l$ qubits (and perform identity on the other $n-l$ qubits), and the following result holds. The proofs of Corollary~\ref{cor:dynamic-model-family} and Corollary~\ref{cor:dynamic-model-local-channel} are deferred to Appendix~\ref{append:dynamical-model}.
\begin{corollary}\label{cor:dynamic-model-local-channel}
Under the assumption that the dynamical models are $l$-local quantum channels, there is an algorithm with dynamic regret $\tilde{O}(L\sqrt{nT\mathcal{P}'}+2^{6l}\sqrt{T})$, where $\mathcal{P}' = \min_{\Phi\in S_{n,l}}\sum_{t = 1}^{T-1} \|\varphi_{t+1}-\Phi(\varphi_{t})\|_\rev{1}\ge1$ and $S_{n,l}$ is the set of all $l$-local quantum channels.
\end{corollary}

\section{Minimizing Adaptive Regret}
In this section, we minimize the adaptive regret in Eq.~\eq{SA-regret} in the non-realizable case, i.e., $b_t$ can be arbitrarily selected.
We follow a meta-algorithm called Coin Betting for Changing Environment (CBCE) proposed \rev{by~\citet{jun2017improved}} with its black-box algorithm $\mathcal{B}$ being the Regularized Follow-the-Leader based algorithm introduced by~\citet{aaronson2018online} for learning quantum states. The algorithm is given in Algorithm~\ref{algo:adaptive}.

\begin{algorithm}
\caption{Adaptive Regret for Quantum Tomography} \label{algo:adaptive}
\begin{algorithmic}[1]
\STATE \textbf{Input:} CBCE from \citet{jun2017improved} as the meta-algorithm, RFTL from \citet{aaronson2018online} as the black-box algorithm.
\STATE Initialize experts $\mathcal{B}_J$, where $\mathcal{B}_J$ is an instance of the RFTL algorithm over a different interval $J$, $J\in \mathcal{J}$ where $\mathcal{J}$ is the set of geometric intervals.
\STATE Initialize a prior distribution over $\mathcal{B}_J$.
\FOR{$t = 1, \ldots, T$}
\STATE Predict among the black-box experts $\mathcal{B}_J$ according to CBCE.
\STATE Update the distribution over the black-box experts $\mathcal{B}_J$ according to CBCE.
\ENDFOR
\end{algorithmic}
\end{algorithm}

The basic idea is that the CBCE meta-algorithm can transfer the regret bound of any black-box algorithm to strongly adaptive regret, and the RFTL algorithm with von Neumann entropy regularization achieves $O(\sqrt{T})$ regret bound which can serve as the black-box algorithm. Formally:

\begin{lemma}
\label{lem:SA-Regret}
Let $E_{1}, E_{2}, \ldots$ be a sequence of two-outcome measurements on an $n$-qubit state presented to the learner, and suppose the loss $\ell_{t}$ is convex, L-Lipschitz, and maps to $[0,1]$. Then Algorithm \ref{algo:adaptive} guarantees strongly adaptive regret $\text{SA-Regret}_T^{\mathcal{A}}(\tau)=O(L \sqrt{n\tau \log(T) })$ for all $T$.
\end{lemma}

\paragraph{$k$-shift regret bound.}
As an application of \rev{Lemma \ref{lem:SA-Regret}}, we derive a $k$-shift regret bound. In the $k$-shift setting, we assume that $\varphi$ can change at most $k$ times, and we derive the following \rev{theorem}. 

\begin{theorem}
\label{thm:k_shift regret}
Under the same setting as \rev{Lemma \ref{lem:SA-Regret}}, if we allow $\varphi$ to change at most $k$ times, we have an $O(L\sqrt{knT \log(T)})$ bound on the $k$-shift regret $R_{k\text{-shift}}^{\mathcal{A}}$:
\begin{align}
    R_{k\text{-shift}}^{\mathcal{A}} = \sum_{t= 1}^T \ell_t(E_t x_t^{\mathcal{A}}) -\min_{1=t_1\le t_2\le \cdots \le t_{k+1}=T+1} \min_{\varphi_1,\varphi_2,\ldots, \varphi_k\in C_n}\sum_{j = 1}^{k}\sum_{t = t_j}^{t_{j+1}-1}\ell_t(E_t\varphi_j).
\end{align}

\end{theorem}

\paragraph{Mistake bound of the $k$-shift setting.}
With \rev{Theorem \ref{thm:k_shift regret}} in hand, when we further assume that the ground truth $\rho_t$ also shifts at most $k$ times, we can derive the following mistake bound. 

\begin{corollary}
\label{cor:k-shift-mistake}
Let $\{\rho_t\}$ be a sequence of $n$-qubit mixed state which changes at most $k$ times, and let $E_1,E_2,\ldots$ be a sequence of two-outcome measurements that are revealed to the learner one by one, each followed by a value $b_t \in[0, 1]$ such that $|\Tr(E_t\rho_t)-b_t| \le \frac{1}{3}\epsilon$. Then there is an explicit strategy for outputting hypothesis states $x_1, x_2, \ldots$ such that $|\Tr(E_t x_t) - \Tr(E_t\rho_t)| > \epsilon$ for at most $O\left(\frac{kn}{\epsilon^2}\log(\frac{kn}{\epsilon^2})\right)$ times.
\end{corollary}

The proofs of the claims in this section are deferred to Appendix~\ref{append:small-proof}.

\section{Numerical Experiments}

\graphicspath{Figures}

In this section, we present numerical results which support our regret bounds.\footnote{\rev{The implementation codes can be found at the link: https://github.com/EstelYang/Adaptive-Online-Learning-of-Quantum-States}} 
Specifically, we conduct experiments for comparing with non-adaptive algorithms and testing our performance of the $k$-shift setting and the path length setting, respectively. All the experiments are performed on Visual Studio Code 1.6.32 on a 2021 MacBook Pro (M1 Pro chip) with 10-core CPU and 16GB memory. We generate random $n$-qubit quantum states using the QuTiP package~\citep{QuTiP,QuTiP2}. 
In the $k$-shift setting, each of the $k$ changes of the quantum state are chosen uniformly at random from all iterations $\{1,2,\ldots,T\}$. In the path length setting, our quantum state evolves slowly following a background Hamiltonian $H$ also generated by QuTiP. At time $t$ we generate a POVM measurement $E_t$ with eigenvalues and eigenvectors chosen uniformly at random from $[0,1]$ and $\ell_{2}$-norm unit vectors, respectively. Then, we apply the measurement, receive loss, and update our prediction using Algorithm~\ref{algo:md+exp} and Algorithm~\ref{algo:adaptive}. For all experiments, the instantaneous regret at each time step is computed using the ground truth state, which changes over each run depending on the setting.

\paragraph{Comparison to non-adaptive algorithms.}
One eminent strength of adaptive online learning of quantum states is its ability to capture the changing states. Previous algorithms such as regularized follow-the-leader (RFTL) perform well for predicting fixed quantum states~\citep{aaronson2018online}, but they are not adaptive when the quantum state evolves. We compare the cumulative regret produced by RFTL and our algorithms in the $k$-shift setting and the path length setting. The results are illustrated as follows.

\begin{figure*}[htbp]
\centering
\subfigure[$k$-shift setting, $k=80$, $n=2$.]{
    \begin{minipage}{.47\linewidth}
    \centering
    \includegraphics[scale=0.47]{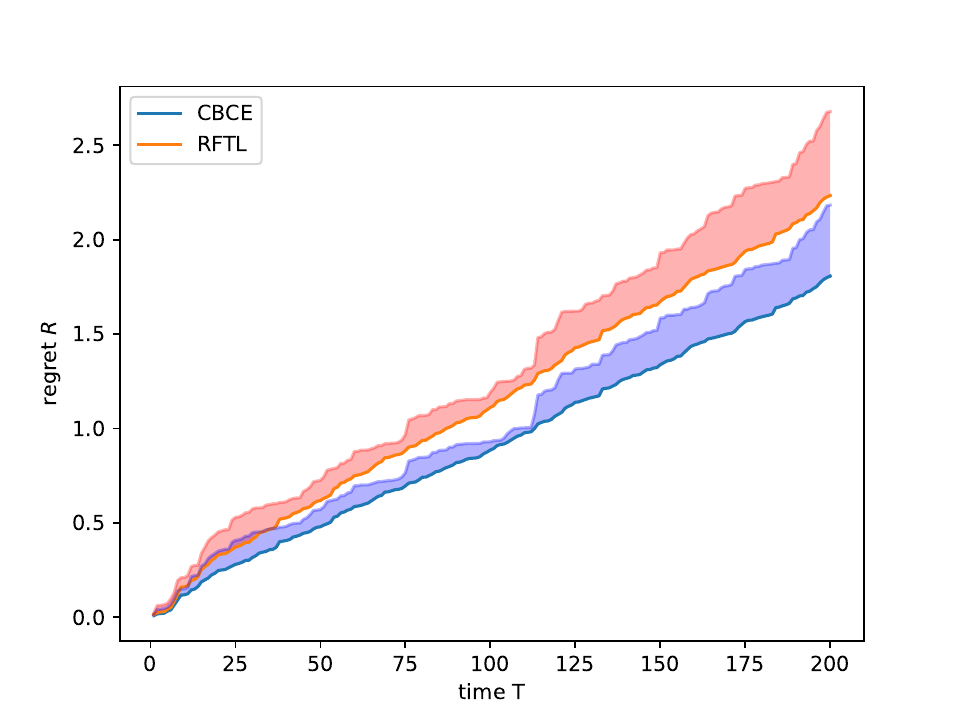}
    \end{minipage}
}
\subfigure[$k$-shift setting, $k=4$, $n=2$.]{
    \begin{minipage}{.47\linewidth}
    \centering
    \includegraphics[scale=0.47]{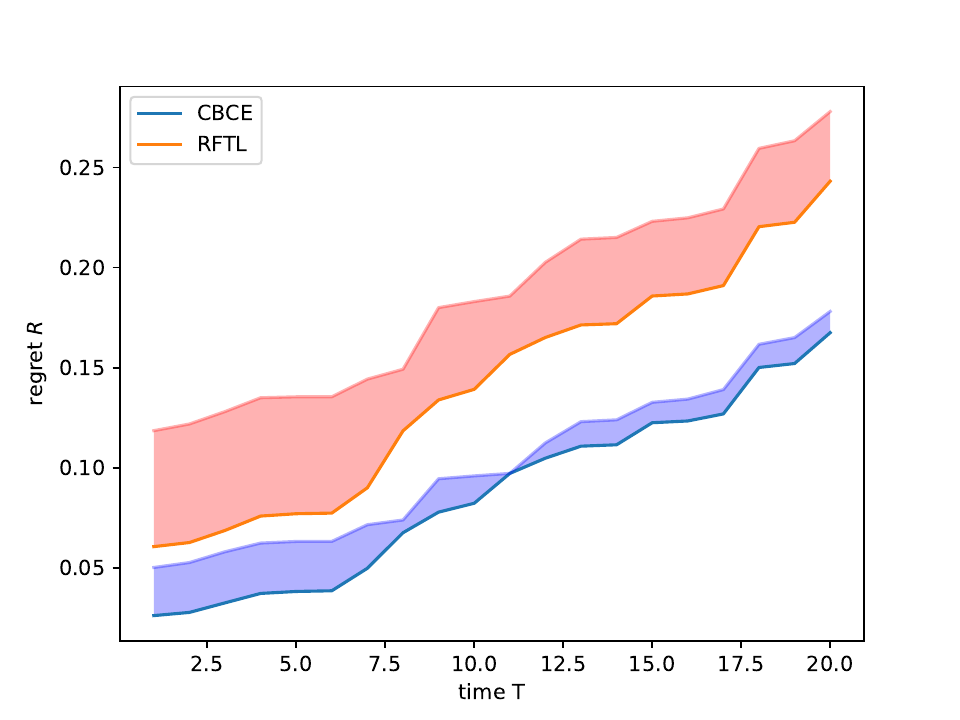}
    \end{minipage}
}
\subfigure[$k$-shift setting, $k=4$, $n=3$.]{
    \begin{minipage}{.47\linewidth}
    \centering
    \includegraphics[scale=0.47]{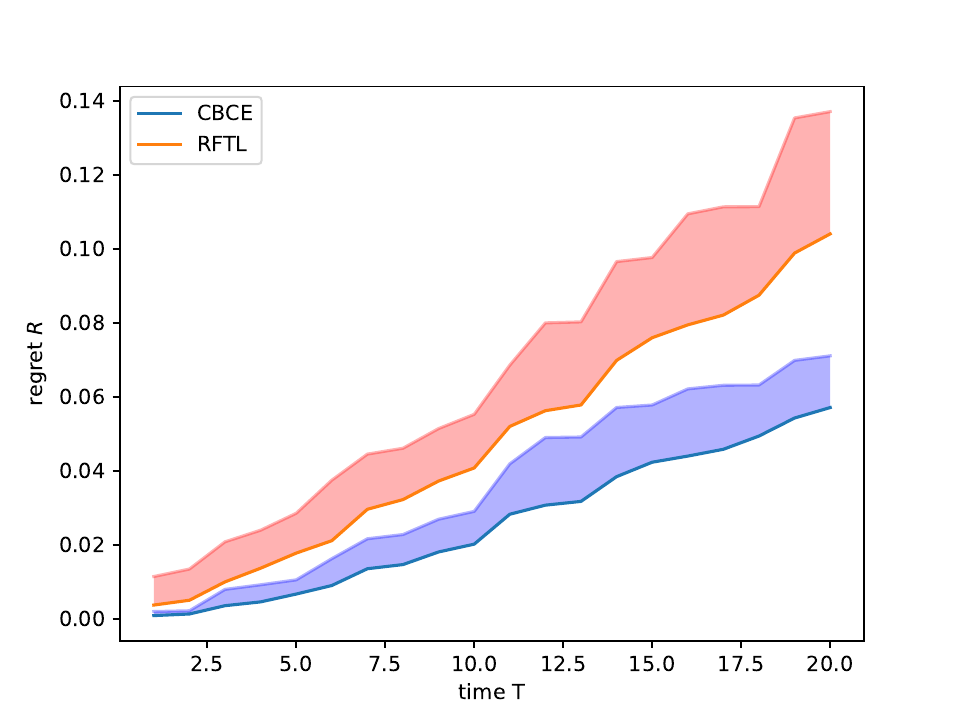}
    \end{minipage}
}
\subfigure[$k$-shift setting, $k=4$, $n=4$.]{
    \begin{minipage}{.47\linewidth}
    \centering
    \includegraphics[scale=0.47]{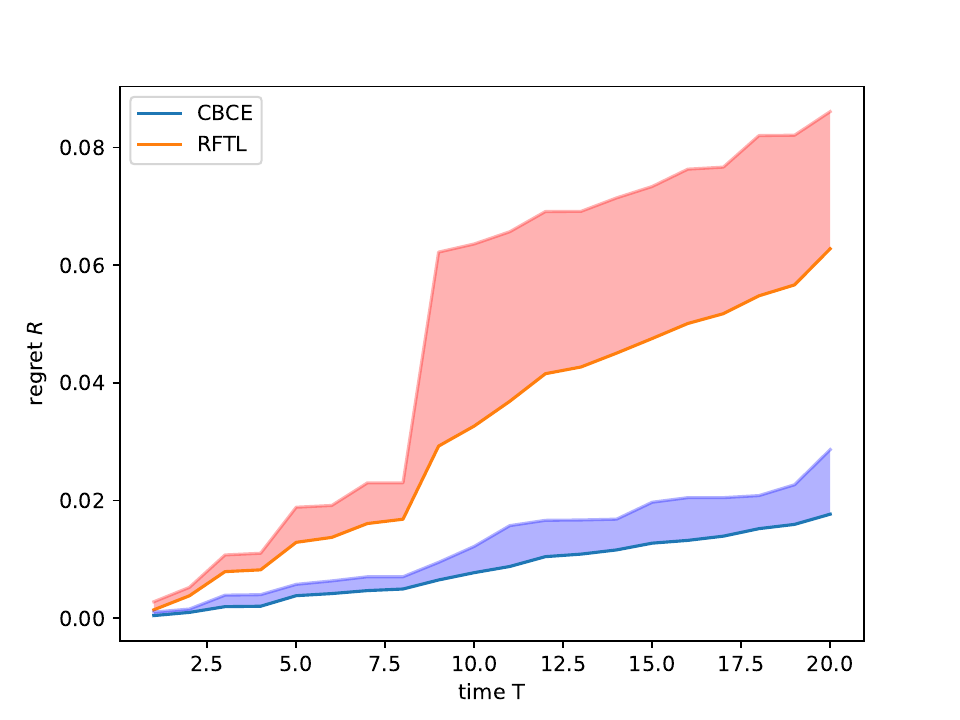}
    \end{minipage}
}
\subfigure[$k$-shift setting, $k=4$, $n=5$.]{
    \begin{minipage}{.47\linewidth}
    \centering
    \includegraphics[scale=0.47]{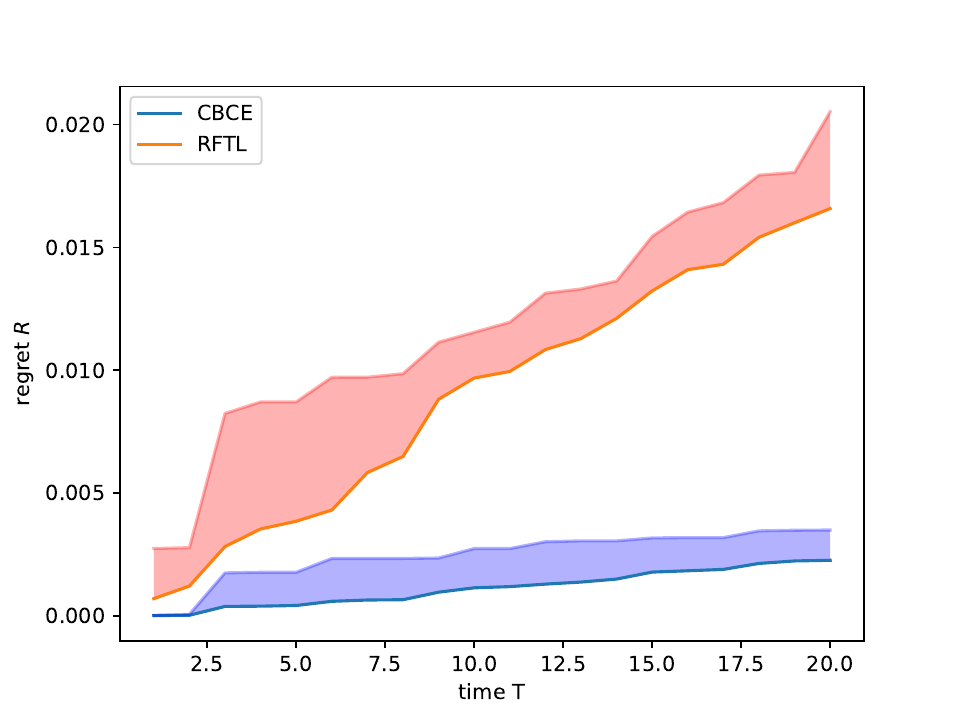}
    \end{minipage}
}    
\subfigure[$k$-shift setting, $k=4$, $n=6$.]{
    \begin{minipage}{.47\linewidth}
    \centering
    \includegraphics[scale=0.47]{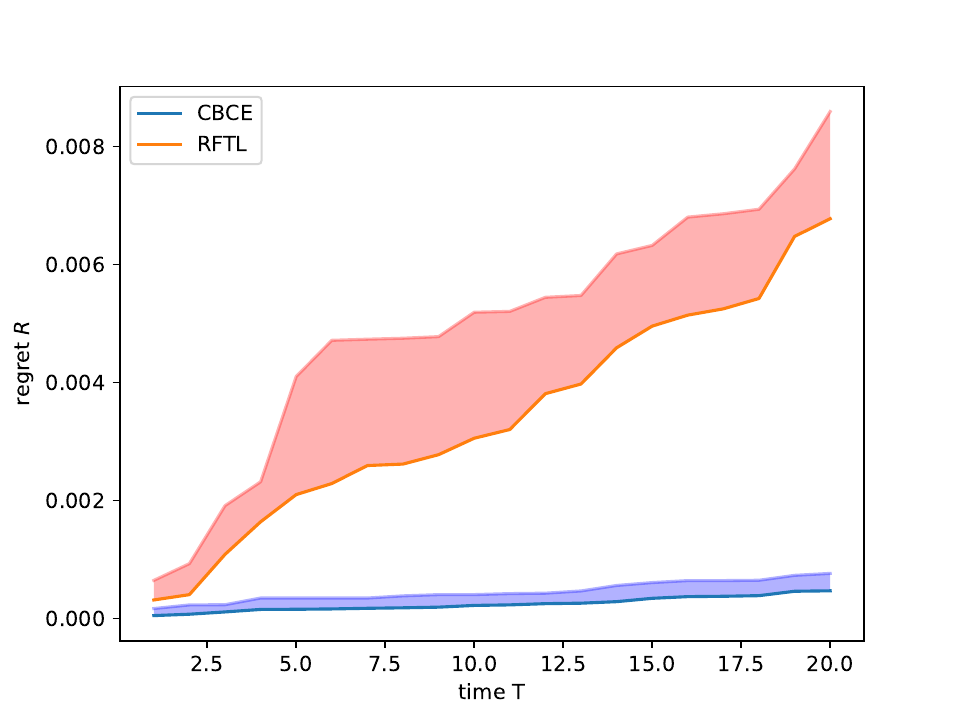}
    \end{minipage}
}
\caption{Comparison between Algorithm~\ref{algo:adaptive} and non-adaptive RFTL in the $k$-shift setting. \rev{The regret attained by our adaptive Algorithm~\ref{algo:adaptive} is denoted by ``CBCE'' (the blue curve with purple shadow), and the regret generated by the non-adaptive algorithm RFTL is denoted by the orange curve with red shadow. The lower solid line stands for the average value of regret over 10 random experiments and the upper line stands for the maximum.}}
\label{fig:non-adaptive-comparison}
\end{figure*}

\begin{figure*}[ht]
\centering
\subfigure[Path length setting, Algorithm~\ref{algo:md+exp}, $n=2$.]{
    \begin{minipage}{.47\linewidth}
    \centering
    \includegraphics[scale=0.47]{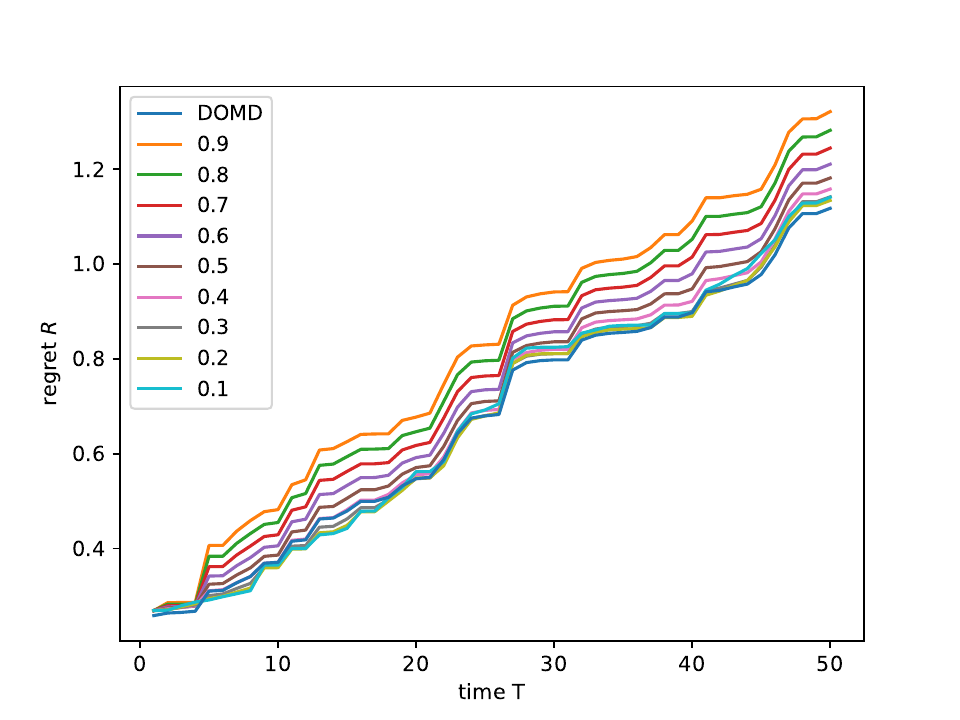}
    \end{minipage}
}
\subfigure[Path length setting, Algorithm~\ref{algo:md+exp}, $n=3$.]{
    \begin{minipage}{.47\linewidth}
    \centering
    \includegraphics[scale=0.47]{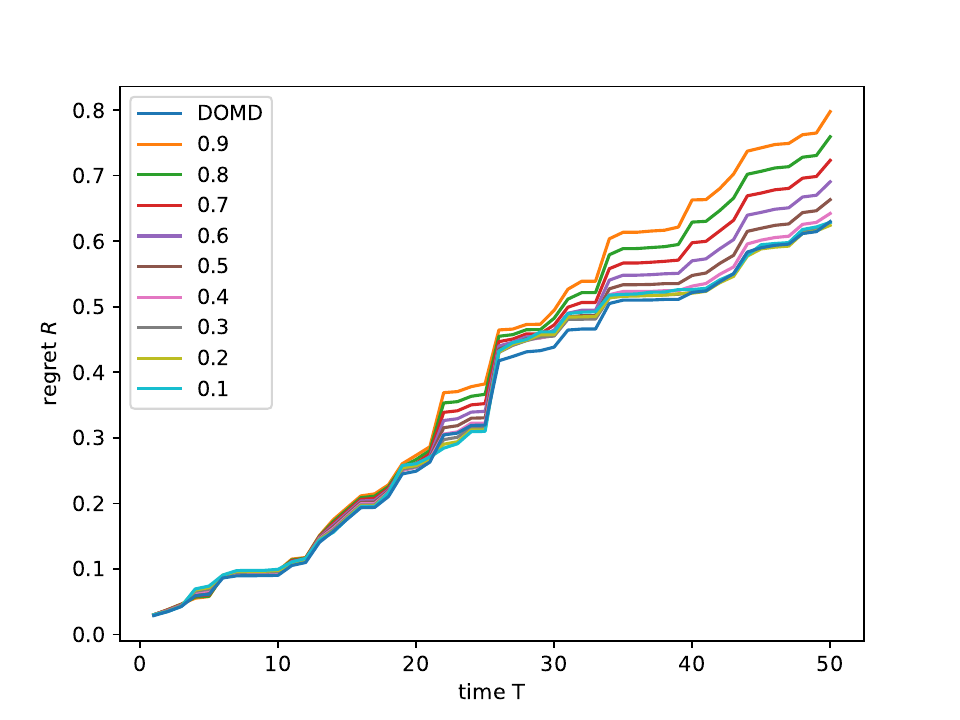}
    \end{minipage}
}
\subfigure[Path length setting, Algorithm~\ref{algo:adaptive}, $n=2$.]{
    \begin{minipage}{.47\linewidth}
    \centering
    \includegraphics[scale=0.47]{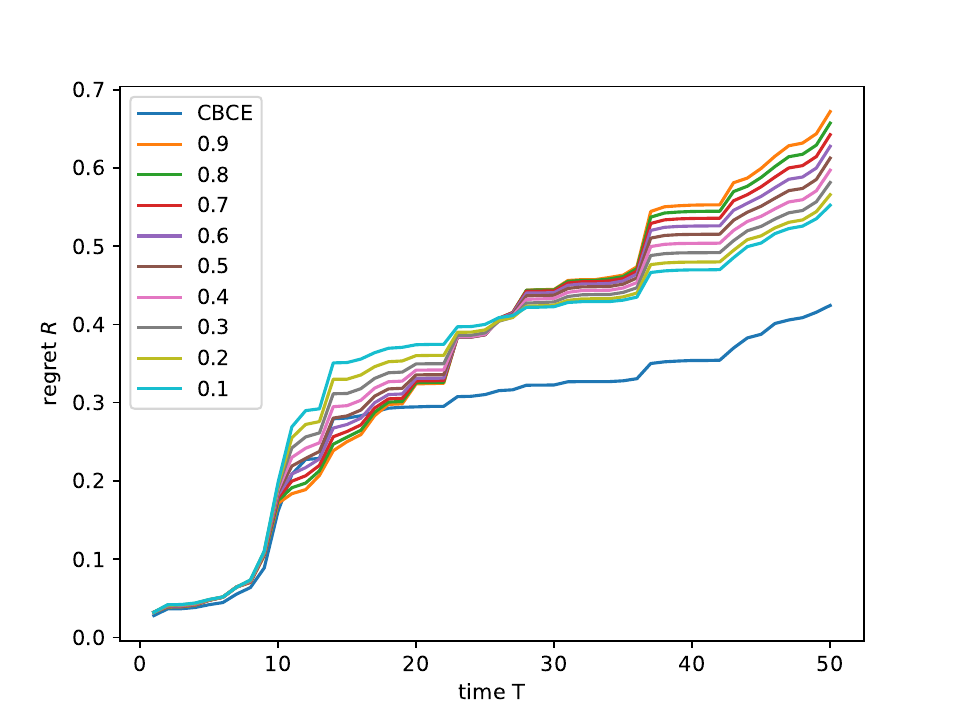}
    \end{minipage}
}
\subfigure[Path length setting, Algorithm~\ref{algo:adaptive}, $n=3$.]{
    \begin{minipage}{.47\linewidth}
    \centering
    \includegraphics[scale=0.47]{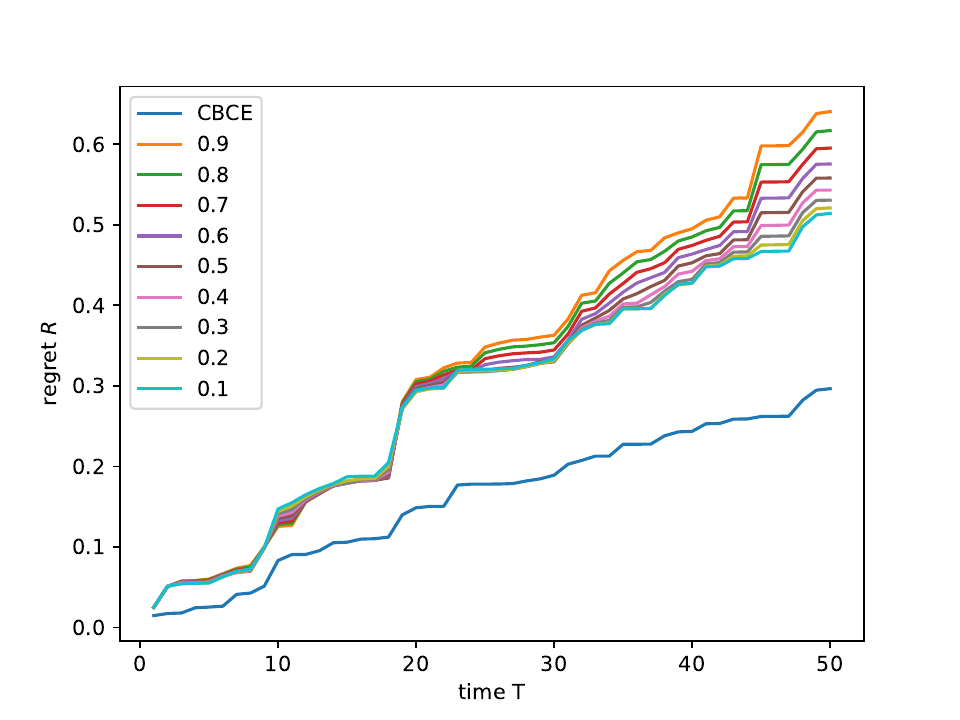}
    \end{minipage}
}
\caption{Comparison between our algorithms and non-adaptive RFTL in the path length setting. \rev{Our algorithms are denoted by ``DOMD'' / ``CBCE", and the non-adaptive algorithm RFTL with different learning rates $\eta$ are denoted by the value of $\eta$.}}
\label{fig:non-adaptive-comparison-path}
\end{figure*}

The advantage of our algorithms in the $k$-shift setting is presented in Figure~\ref{fig:non-adaptive-comparison}. In Subfigure(a), we plot the long term changes of the cumulative regret. The two curves have similar trends but CBCE achieves smaller regret. 
In Subfigure(b)-(f), we choose a small time scale $T =20$ to display the details. It is clearly shown that the distances between two curves grow larger as $T$ increases, which means our algorithm maintains the advantage in the long run. 

Figure~\ref{fig:non-adaptive-comparison-path} shows the advantage of our algorithms in the path length setting. In Subfigure(a) and (b), we plot the comparison between our Algorithm~\ref{algo:md+exp} (denoted by ``DOMD'') and the non-adaptive algorithm RFTL \rev{with different learning rates}. It can be seen that in a single run, our algorithm attains the minimum regret and is comparable to RFTL with the best learning rate. Indeed, given the similarity of RFTL and OMD, we expect RFTL with the optimal learning rate to do as well as DOMD in practice. However, since the optimal learning rate for OMD depends on the path length, it is unknown a priori, and non-adaptive methods require more tuning to achieve good performance. In Subfigure(c) and (d), we plot the comparison between our Algorithm~\ref{algo:adaptive} (denoted by ``CBCE'') and the non-adaptive RFTL with different $\eta$. There is a significant gap between the regret of CBCE and RFTL with the best learning rate from the grid search, demonstrating the interesting phenomenon that empirically, CBCE with adaptive regret guarantees also does well in the path length setting.


\paragraph{Adaptive regret bound on the $k$-shift setting.}
We also conduct experiments to verify our adaptive regret bound $O(\sqrt{knT\log T})$ in the $k$-shift setting (\rev{Theorem~\ref{thm:k_shift regret}}). Results are shown in Figure~\ref{fig:k-shift}.
\begin{figure}[H]
\centering
\hspace{-2em}
\subfigure[Average Regret $R_{\text{aver}}$]{
\begin{minipage}{0.47\linewidth}
    \centering
    \includegraphics[scale=0.45]{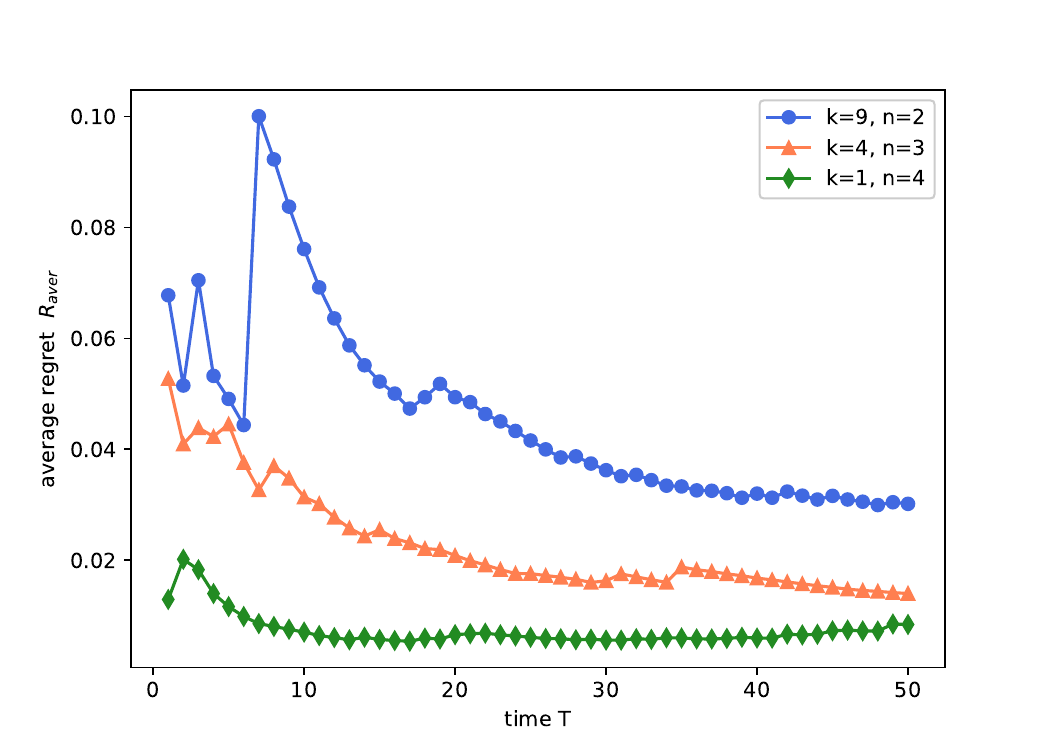}
\end{minipage}
}
\hspace{1em}
\subfigure[Ratio $C= R_{k\text{-shift}}/\sqrt{knT\log T}$]{
\begin{minipage}{0.47\linewidth}
    \centering
    \includegraphics[scale=0.45]{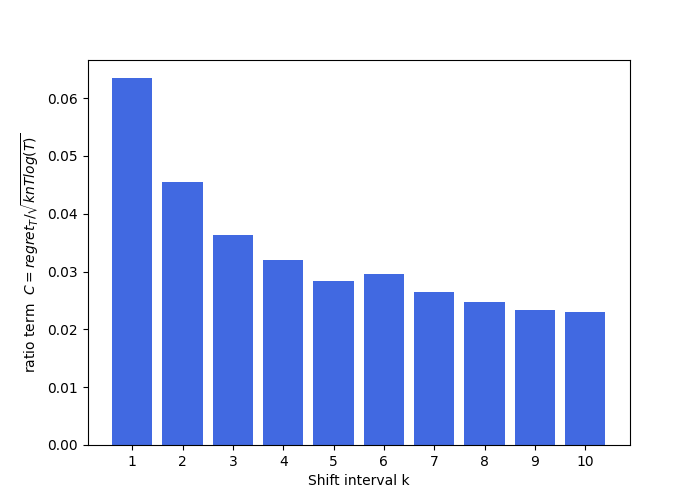}
    \end{minipage}
}
\caption{Adaptive regret bounds in the $k$-shift setting.}
\label{fig:k-shift}
\vspace{-1mm}
\end{figure}

In Subfigure(a), we plot the average regret $R_{\text{aver}}=\text{regret}_T /T$ as $T$ grows. Every data-point was executed by 100 random experiments and taking the maximum regret, since regret is an upper bound. Under different parameter choices $(k,n)$, $R_{\text{aver}}$ always converges to 0 when $T$ increases, implying a sublinear regret $R_{k\text{-shift}}$. More concretely, curves with larger $k$ has larger $R_{\text{aver}}$ and converges more slowly, consistent with the intuition that more shifts make the learning task more difficult.

In Subfigure(b), to verify the precise bound we plot the ratio $C = \text{regret}_T / \sqrt{knT\log T}$. We choose $n=2, k=10, T=200$, and repeat the experiments 50 times randomly. We then plot the maximum of the regret constant $C$ on every interval. As the figure shows, the ratio is limited by a constant $C \leq 0.07$. This strongly supports our upper bound $O(\sqrt{knT\log T})$ on $\text{regret}_T$  with every fixed $n$. 

\vspace{-1mm}
\paragraph{Dynamic regret bound on the path length setting.}
In addition, we conduct experiments in terms of average regret $R_{\text{aver}}$ and ratio term $C_{\text{path}} =R_{\text{path}}/\sqrt{T(n+\log T)\mathcal{P}}$ to testify our dynamic regret bound $O(L\sqrt{T(n+\log T)\mathcal{P}})$ in the path length setting (Theorem~\ref{thm:dynamic}). Results are shown in Figure~\ref{fig:path-length}.
\begin{figure}[htbp]
\centering
\hspace{-2em}
\subfigure[Average Regret $R_{\text{aver}}$]{
    \begin{minipage}{0.47\linewidth}
    \centering
    \includegraphics[scale=0.45]{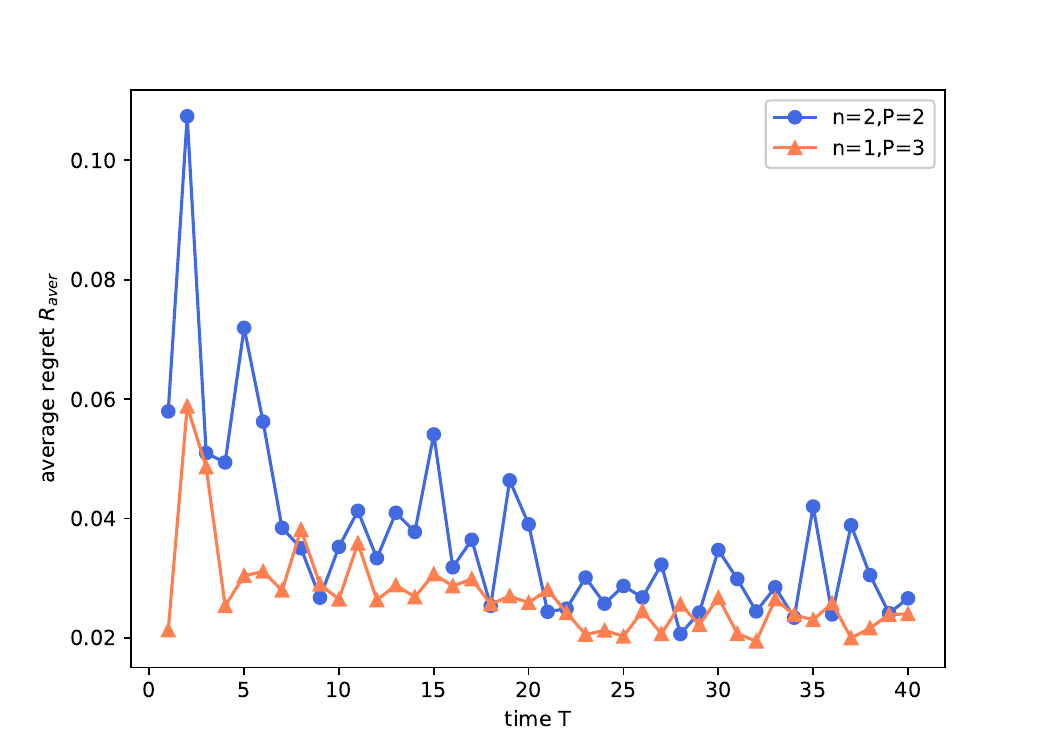}
    \end{minipage}
}
\hspace{1em}
\subfigure[Ratio $C=R_{\text{path}}/\sqrt{T(n+\log T)\mathcal{P}}$]{
    \begin{minipage}{0.47\linewidth}
    \centering
    \includegraphics[scale=0.45]{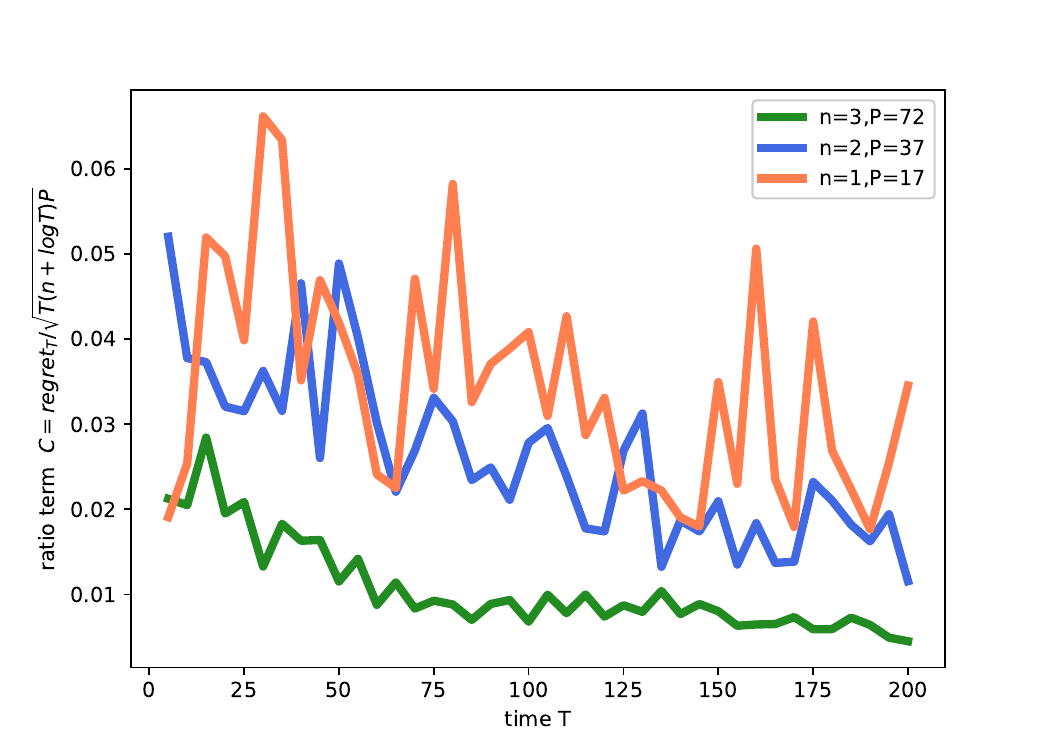}
    \end{minipage}
}
\caption{Dynamic regret bound on the path length setting.}
\label{fig:path-length}
\end{figure}

\vspace{-2mm}
In Subfigure(a), we 
plot the change of average regret $R_{\text{aver}}=\text{regret}_T /T$. We observe that as $T$ grows, $R_{\text{aver}}$ converges to 0 with different pairs of $(n,\mathcal{P})$, indicating that $R_{\text{path}}$ is sublinear with respect to $T$. In addition, the blue curve representing the average regret of a 2-qubit quantum state appears higher than the orange curve representing the average regret of a 1-qubit quantum state, consistent with the fact that our regret bound depends polynomially on the number of qubits.

In Subfigure(b), we choose different scales of $T$ to demonstrate how the ratio $C$ evolves in the long term. For every fixed $T$, we execute 100 random experiments and take the maximum ratio term $C$. Although there are noticeable fluctuations, it can be seen that the curve is non-increasing and only fluctuates within a fixed interval of 0 to 0.065, which shows that $C$ has a constant upper bound $C'=0.065$. In conclusion,  $R_{\text{path}}/\sqrt{T(n+\log T)\mathcal{P}}$ is within the range of a small positive constant.

\section{Conclusions}
In this paper, we consider online learning of changing quantum states. We give algorithms with adaptive and dynamic regret bounds that are polynomial in the number of qubits and sublinear in the number of measurements, and verify our results with numerical experiments. 

\rev{Our work leaves a couple of open questions for future investigation:
\begin{itemize}
\item Lower bound in $T$ and path length $P$ are known in \cite{zhang2018adaptive}, but there is no lower bound in $n$ to our knowledge. Can we prove an $\Omega(\sqrt{n})$ lower bound for dynamic regret or adaptive regret of online learning of quantum states? We noticed a recent work \cite{raza2024online} that proposed a lower bound in terms of $n$ for quantum channels. It would be intriguing in future work to explore whether these techniques can be adapted to the quantum state scenario.
\item Although our algorithms do not care how the loss function is revealed, in practice, a standard way is to measure several copies of the state at each step, so the copies need to follow the same dynamics. Accounting for different state evolution for different state copies is left as future work.
\end{itemize}
}
 
\section*{Acknowledgments}
TL, XW, and RY were supported by the National Natural Science Foundation of China (Grant Numbers 92365117 and 62372006), and The Fundamental Research Funds for the Central 
Universities, Peking University.

\bibliographystyle{quantum}
\bibliography{reference}


\newpage
\appendix

\input{appendix}

\end{document}

%% file: appendix.tex
\section{Proof of Dynamic Regret Bounds for Quantum Tomography}\label{append:proof-thm-dynamic}
We prove Theorem \ref{thm:dynamic} in this section. 

We use the online mirror descent (OMD) algorithm (similar to Algorithm 1 in \citet{aaronson2018online}) to replace the online gradient descent (OGD) expert in \citet{zhang2018adaptive}. We denote $G_R$ to be an upper bound on the spectral norm of $\nabla_t= \ell’_t (\Tr(E_t x_t)) E_t$, and $D_R=\sqrt{\max_{x,y\in \mathcal{K}} \{R(x)-R(y)\}}$ to be the diameter of $\mathcal{K}$ with respect to the function $R$. The main motivation is that the dimension now is $2^n$ and we need logarithmic dependence on dimension for our regret bound, where $\ell_2$ norm (used by OGD) is not the most appropriate choice and fails to control $G_R$.

Instead, the OMD algorithm with von Neumann entropy has a polynomial upper bound on both $D_R$ and $G_R$, and the only question left is how to incorporate the path length into the analysis of OMD to achieve dynamic regret. We observe that such incorporation is natural, due to the essential similarity between OGD and OMD. We bound the path distance $\|\varphi_t-\varphi_{t+1}\|$ in nuclear norm, and bound the gradient of the regularization $\nabla R(x)$ by spectral norm.

A potential issue is that the von Neumann entropy contains logarithmic terms, which might cause the gradient to explode when $\lambda_k$ is very small. To mitigate this issue, we mix $x$ with $\frac{I}{T 2^n}$: $\tilde{x}=(1-\frac{1}{T})x+\frac{I}{T 2^n}$, for both $x_t$ and $\varphi_t$. Notice that this  mixing does not hurt the regret bound: mixing such a small fraction of the identity only increases the regret by  $O(LT \frac{1}{T})=O(L)$, since the losses are $L$-Lipschitz. In addition, the path length will only decrease, since we are working over a smaller domain. Henceforth we will consider the modified $\mathcal{K}=(1-\frac{1}{T})C_n+\frac{1}{T 2^n} I$ throughout the proof. 

Our proof uses tools in convex analysis for Hilbert spaces. Since all matrices used in Algorithm 
\ref{algo:algomd} are Hermitian (by Proposition \ref{lem:pd}), it suffices to consider the Hilbert space of Hermitian matrices equipped with the trace inner product. Note that this is a real Hilbert space: the vector space of Hermitian matrices only allows scalar multiplication over the field of reals.

We first present some basic results for the von Neumann Entropy. For completeness, we give the following definitions from \citet{operator_theory}; for a more comprehensive treatment of the topic of analysis in Hilbert spaces, see also \citet{operator_theory}.


\begin{definition}[Gateaux differentiability]
Let $C$ be a subset of a Hilbert space $\mathcal{H}$, and  $f:C\rightarrow [-\infty, \infty]$. Let $x \in C$ such that for each $y\in \mathcal{H}$, $\exists \alpha > 0$ such that $[x, x+\alpha y] \subset C$.Then $f$ is Gateaux differentiable at $x$ if there exists a bounded linear operator $Df(x)$ called the Gateaux derivative of $f$ at $x$, such that 
$$
D f(x)y = \lim_{\alpha \downarrow 0} \frac{f(x + \alpha y) - f(x)}{\alpha}.
$$
\end{definition}

\begin{definition}[Directional derivatives]
For a Hilbert space $\mathcal{H}$, let $f: \mathcal{H}\rightarrow (-\infty, \infty]$ be proper, and $x\in \text{dom} f,y\in\mathcal{H}$.The directional derivative of $f$ at $x$ in the direction $y$ is
$$
f'(x;y) = \lim_{\alpha \downarrow 0} \frac{f(x + \alpha y) - f(x)}{\alpha},
$$
if the limit exists on the extended real line. 
\end{definition}

It is well-known that $f$ is Gateaux differentiable at $x$ if and only if all directional derivatives at $x$ exist and they form a bounded linear operator. By the Riesz-Frechet representation theorem, let $\langle \cdot, \cdot \rangle$ denote the Hilbert space inner product, there exists a unique $\nabla f(x) \in \mathcal{H}$ such that for each $y\in \mathcal{H}$, 
$$
f'(x;y) = \langle y, \nabla f(x)\rangle.
$$

\begin{lemma}[Lemma VI.4 in Ref.~\cite{gateaux}]\label{lem:gateaux}
The von Neumann entropy $R(x) = \Tr(x\log x)$ for Hermitian PD matrices $x$ is Gateaux differentiable, and the Gateaux derivative is 
$$
\nabla R(x) = \log x + I.
$$
\end{lemma}

The quantum analog of the classical Bregman divergence is therefore defined as
$$
B_R(x\|y) = R(x) - R(y) - \nabla R(y) \bullet (x-y).
$$
The above expression is also known as the quantum relative entropy of $x$ with respect to $y$, the quantum information divergence, and the von Neumann divergence \citep{quantum_divergence}.

\begin{proposition}\label{lem:pd}
The variable $y_t$ in Algorithm~\ref{algo:algomd} is always Hermitian PD. 
\end{proposition}
\begin{proof}
By the definition of $y_t$ and $\nabla R(\cdot)$, we have
$
\log(y_{t+1}) = \log(x_{t-1}) - \eta \nabla_{t-1}.
$
Recall that the loss functions are of the form $\ell_t(\rho)  = \ell(E_t \bullet \rho)$. Since $E_t$ is Hermitian, $\nabla_{t-1} = \ell'(\cdot)  E_{t-1}$ is also Hermitian. The lemma follows by definition.
\end{proof}

\begin{proposition}\label{prop:prop_2}
For all Hermitian PSD matrices $x, y$, $B_R(x\|y)\in \mathbb{R}$. For all $x\in \mathcal{K}$ and all Hermitian PD $y$, $B_R(x\|y)\ge 0$, and for all $\varphi_1\in \mathcal{K}$ and $x_1$ as defined in Algorithm~\ref{algo:algomd}, $B_R(\varphi_1\|x_1)\le \max_{x, y \in \mathcal{K}} \left\{R(x) - R(y)\right\}$.
\end{proposition}

\begin{proof}
The first claim follows from the definition of $R$ and $\nabla R$. Since both $x, y$ are both positive definite, by Fact 3 in \citet{quantum_divergence} (Klein's inequality), $B_R(x\|y)\ge 0$. For the last claim, $\varphi_1 \in \mathcal{K}$ be any element, to show $B_R(\varphi_1\|x_1) = R(\varphi_1) - R(x_1) - \nabla R(x_1) \bullet (\varphi_1 - x_1) \le \max_{x, y\in \mathcal{K}} R(x) - R(y)$, we only need to show $\nabla R(x_1) \bullet (\varphi_1 - x_1) \ge 0$. By direct computation,
\begin{align*}
    \nabla R(x_1) \bullet (\varphi_1 - x_1) = (I + \log(2^{-n} I))\bullet (\varphi_1 - 2^{-n}I) 
    = (-n\log(2) I)\bullet (\varphi_1 - 2^{-n}I) = 0,
\end{align*}
since $\Tr(\varphi_1) = \Tr(2^{-n}I) = 1$.
\end{proof}


\begin{proposition}[Euler's inequality]\label{prop:euler} Let $x_t, y_t$ be as defined in Algorithm \ref{algo:algomd}, then for any $\varphi \in \mathcal{K}$, we have 
$$B_R(\varphi\|x_t) \le B_R(\varphi\|y_t).
$$
\end{proposition}
\begin{proof}
For $x\in \mathcal{K}$, 
$$
B_R(x\|y) = \Tr(x\log x - x\log y - x + y).
$$
By Lemma \ref{lem:gateaux}, $\Tr(x
\log x)$ is Gateaux differentiable on the set of positive definite Hermitian matrices. Since $\Tr(x\log y)$ and $\Tr(x)$ are both Gateaux differentiable, $B_R(x\|y)$ is also Gateaux differentiable as a function of $x$. By definition, $x_t = \text{argmin}_{x\in \mathcal{K}} B_R(x\|y_t)$, and therefore for $\alpha \in (0, 1]$ and any $\varphi\in \mathcal{K}$,
$$B_R(x_t\|y_t) \le B_R((1-\alpha)x_t + \alpha \varphi\|y_t).
$$
\rev{Dividing by $\alpha$ and} taking $\alpha \downarrow 0$, we have 
$$
\nabla B_R(x_t\|y_t)\bullet (\varphi - x_t) \ge 0.
$$
Expanding the derivative, $(\nabla R(x_t) - \nabla R(y_t)) \bullet (\varphi - x_t) \ge 0$. Observe that we have the decomposition
$$
    ( \nabla R(y_t)-\nabla R(x_t))\bullet ( \varphi -x_t ) =B_R(x_t||y_t)-B_R(\varphi||y_t)+B_R(\varphi||x_t), \notag
$$
which means $B_R(\varphi||y_t)\ge B_R(x_t||y_t)+B_R(\varphi||x_t)\ge B_R(\varphi||x_t)$, because $B_R(x_t||y_t)\ge 0$ by Proposition \ref{prop:prop_2}.
\end{proof}

We now prove the main lemma which gives a dynamic regret bound of the OMD algorithm, when the learning rate $\eta$ is chosen optimally.

\begin{lemma}
\label{thm:dynamic_regret}
Assume the path length $\mathcal{P}=\sum_{t=1}^{T-1} \|\varphi_t-\varphi_{t+1}\|_\rev{1}\ge 1$ where $\|\cdot\|_\rev{1}$ is the nuclear norm, the dynamic regret of Algorithm~\ref{algo:algomd} is upper bounded by $O(L \sqrt{T(n+\log T) \mathcal{P}})$, if we choose $\eta=\sqrt{\frac{\mathcal{P}(n+\log T)}{T L^2}}$.
\end{lemma}
\begin{proof}
We follow the standard analysis of OMD, assuming $\ell_t$ are $L$-Lipschitz. By convexity of loss $\ell_t$, we have that
\begin{align}
   \ell_t (\Tr(E_t x_t))-\ell_t(\Tr(E_t \varphi_t))
   \le \ell_t'(\Tr(E_tx_t))(\Tr(E_tx_t) - \Tr(E_t \varphi_t))
   \rev{=} \nabla_t \bullet (x_t-\varphi_t). \label{eq:convexity}
\end{align}

The following property of Bregman divergences follows from the definition:
for any Hermitian PSD matrices $x, y, z$,
\begin{align}
    (x-y) \bullet (\nabla R(z)- \nabla R(y)) = B_R(x\|y)- B_R(x\|z)+B_R(y\|z).\label{eq:bregman}
\end{align}
Now, take $x = \varphi_t$, $y = x_t$, and $z = y_{t+1}$ and combine \eqref{eq:convexity} and \rev{ \eqref{eq:bregman}}. Writing $\ell_t(x) = \ell_t(\Tr(E_t x))$, we have
\begin{equation}
\begin{split}
    \label{eqn:loss_ineq}
    & \ell_t(x_t)-\ell_t(\varphi_t) \notag \\
    \le & \frac{1}{\eta}(B_R(\varphi_t\|x_t)- B_R(\varphi_t\|y_{t+1})+B_R(x_t\|y_{t+1})) \notag \\
    \le & \frac{1}{\eta}(B_R(\varphi_t\|x_t)- B_R(\varphi_t\|x_{t+1})+B_R(x_t\|y_{t+1})) \notag \\
    = & \frac{1}{\eta}\big(B_R(\varphi_t\|x_t)-B_R(\varphi_{t+1}\|x_{t+1})+B_R(\varphi_{t+1}\|x_{t+1}) 
        - B_R(\varphi_t\|x_{t+1})+B_R(x_t\|y_{t+1})\big),
    \notag
\end{split}
\end{equation}

where the second inequality is due to Proposition \ref{prop:euler}.

The following lemma gives an upper bound on one of the quantities in Eq.~(\ref{eqn:loss_ineq}), which we will use next in the proof. In vanilla OMD, this quantity is upper bounded using the fact that the Bregman divergence can be interpreted as a distance measured by a certain local norm; however, under complex domain, this fact is no longer immediate from the Mean Value Theorem. We instead prove the upper bound by using properties of the dual space updates $y_{t+1}$.
\begin{lemma}
Let $G_R$ be an upper bound on $\|\nabla_t\|$, then $B_R(x_t\|y_{t+1}) \le \frac{1}{2}\eta^2 G_R^2.$
\end{lemma}
\begin{proof}
First, note that by definition, $y_{t+1}$ satisfies
\begin{equation}\label{eqn:dual_update}
\log(y_{t+1}) = \log(x_t) - \eta \nabla_t.
\end{equation}
Then we have $
y_{t+1} = \exp(\log(x_t) - \eta \nabla_t),
$
and by the Golden-Thompson inequality,
\begin{align}
    \Tr(y_{t+1}) \le \Tr(\exp(\log(x_t))\exp(-\eta\nabla_t))  \le \Tr(x_t\exp(-\eta\nabla_t)).\label{eqn:golden_thompson}
\end{align}

Expanding the Bregman divergence, we have
\begin{align*}
    B_R(x_t\|y_{t+1})
    = & \Tr(x_t\log(x_t)) - \Tr(y_{t+1}\log(y_{t+1})) - (I + \log(y_{t+1})) \bullet (x_t - y_{t+1})\\
    =&\Tr(x_t\log(x_t)) - \Tr(x_t\log(y_{t+1})) - 1 + \Tr(y_{t+1})\\
    =& \Tr(x_t\log(x_t)) - \Tr(x_t(\log(x_t) - \eta \nabla_t)) - 1 + \Tr(y_{t+1}) \tag{by \eqref{eqn:dual_update}}\\
    =& \eta\Tr(x_t\nabla_t) -1 + \Tr(y_{t+1}) \\
    \le & \eta\Tr(x_t\nabla_t) -1 +\Tr(x_t\exp(-\eta\nabla_t)). \tag{by \eqref{eqn:golden_thompson}}
\end{align*}
By definition, the matrix exponential can be written as an infinite sum,
$$\exp(-\eta \nabla_t) = \sum_{k=0}^\infty \frac{1}{k!}(-\eta \nabla_t)^k.
$$
Using this expression in the last term of the inequality,
\begin{align*}
    \Tr(x_t\exp(-\eta\nabla_t)) = 1 - \eta\Tr(x_t\nabla_t) + \frac{1}{2} \eta^2 \Tr(x_t\nabla_t^2) + \sum_{k=3}^\infty \frac{1}{k!} \Tr(x_t(-\eta \nabla_t)^k).
\end{align*}
Now we proceed to show that $\sum_{k=3}^\infty \frac{1}{k!} \Tr(x_t(-\eta \nabla_t)^k) \le 0$ if $\eta$ is sufficiently small.
\begin{align*}
    \sum_{k=3}^\infty \frac{1}{k!} \Tr(x_t(-\eta \nabla_t)^k) 
    = & \sum_{k=1}^\infty \frac{1}{(2k+1)!} \Tr(x_t(-\eta \nabla_t)^{2k+1}) + \frac{1}{(2k+2)!} \Tr(x_t(-\eta \nabla_t)^{2k+2})\\
    = & \sum_{k=1}^\infty  \frac{\eta^{2k+1}}{(2k+1)!} \Tr\big(x_t\big(  -\nabla_t^{2k+1} + \frac{\eta}{2k+2} \nabla_t^{2k+2}\big)\big).
\end{align*}
Let $\nabla_t = VQV^\dagger$ be the eigendecomposition of $\nabla_t$. Then
\begin{align*}
    -\nabla_t^{2k+1} + \frac{\eta}{2k+2} \nabla_t^{2k+2} = VQ^{2k+1}(-I +\frac{\eta}{2k+2}Q)V^\dagger.
\end{align*}

Since $\eta < \frac{1}{2L}$, and $\|\nabla_t\| \le L$,
\begin{align*}
    -I +\frac{\eta}{2k+2}Q \preceq 0, \text{and}\ \  \Tr\big(x_t\big(  -\nabla_t^{2k+1} + \frac{\eta}{2k+2} \nabla_t^{2k+2}\big)\big) \le 0.
\end{align*}

Because the summands are nonpositive, we conclude that $\sum_{k=3}^\infty \frac{1}{k!} \Tr(x_t(-\eta \nabla_t)^k) \le 0$. The lemma follows from applying the generalized Cauchy-Schwartz inequality (see~\citealp[Claim 14]{aaronson2018online} and~\citealp[Exercise IV.1.14, page 90]{bhatia1997matrix}),
\begin{align*}
    B_R(x_t\|y_{t+1}) \le \frac{1}{2} \eta^2 \Tr(x_t\nabla_t^2) \le \frac{\eta^2}{2} \|x_t\|_\rev{1}\|\nabla_t\|^2 \le \frac{\eta^2G_R^2}{2}.
\end{align*}
\end{proof}
Summing over $t$ and telescoping and observing any eigenvalue is lower bounded by $\frac{1}{2^n T}$, we have
\begin{align}
    \label{eqn:bregman_distance_ineq}
    & \textbf{Regret} \notag \\
    \le & \frac{1}{\eta}(B_R(\varphi_1\|x_1)-B_R(\varphi_T\|x_T)) \notag \\  
    & + \frac{1}{\eta} \sum_{t=1}^{T-1} |B_R(\varphi_t\|x_{t+1})-B_R(\varphi_{t+1}\|x_{t+1})| +\frac{1}{\eta} \sum_{t=1}^{T-1} B_R(x_t\|y_{t+1}) \notag \\
    \le & \frac{2(D_R^2+2n+\log T)}{\eta}+ \frac{\eta T G_R^2}{2} + \frac{1}{\eta} \sum_{t=1}^{T-1} |B_R(\varphi_t\|x_{t+1})-B_R(\varphi_{t+1}\|x_{t+1})| \notag \\
    \le & \frac{2(D_R^2+2n+\log T)}{\eta}+\frac{\eta T G_R^2}{2} + \frac{1}{\eta} \sum_{t=1}^{T-1} |R(\varphi_t)-R(\varphi_{t+1}) -\nabla R(x_{t+1})\bullet (\varphi_t-\varphi_{t+1})|\notag \\
    \le & \frac{2(D_R^2+2n+\log T)}{\eta}+\frac{\eta T G_R^2}{2} \notag \\  
    & + \frac{1}{\eta} \sum_{t=1}^{T-1} \|\nabla R(x_{t+1})\|\| (\varphi_t-\varphi_{t+1})\|_\rev{1} + \frac{1}{\eta} \sum_{t=1}^{T-1} |R(\varphi_t)-R(\varphi_{t+1})|.
\end{align}

From \citet{aaronson2018online} we already know $D_R=O(\sqrt{n})$ and $G_R=O(L)$ due to the use of the von Neumann entropy. We now compute an upper bound on $\|\nabla R(x_{t+1})\|$, suppose $x_{t+1} = VQV^\dagger$ is the eigendecomposition of $x_{t+1}$:
\begin{align*}
    \|\nabla R(x_{t+1})\| = \|I + \log(x_{t+1})\| = \|I + V\log(Q)V^\dagger\| \le n\log(2) + \log(T).
\end{align*}
The inequality is due to the fact that for $x_{t+1}\in \mathcal{K}$, the eigenvalues of $x_{t+1}$ are in $[\frac{1}{2^nT}, 1]$.

For the last term in the regret bound, the Fannes–Audenaert inequality \citep{Audenaert_2007} states that
\begin{align*}
    |R(\varphi_t) - R(\varphi_{t+1})| \le \frac{n}{2}\|\varphi_t - \varphi_{t+1}\|_\rev{1} + H(\frac{1}{2}\|\varphi_t - \varphi_{t+1}\|_\rev{1}),
\end{align*} 
where $H(p) = -(p\log(p) + (1-p)\log(1-p))$ is the Shannon entropy.

We try to bound $H(p)$ by a linear function $g(p)=\lambda p+c$. 
Formally we have the following lemma:
\begin{lemma}
For any $\lambda>0$, we have for $p \in (0,1)$,
$$ H(p) \leq \lambda p + \log \frac{e^\lambda+1}{e^{\lambda}}. $$ 
\end{lemma}
\begin{proof}
Take the derivative of $g(p)-H(p)$, we have that the minimal happens when $\lambda=\log \frac{1-p}{p}$ which is $p=\frac{1}{e^{\lambda}+1}$, and it's easy to verify that the minimal value of $g(p)-H(p)$ equals
$$\frac{\lambda}{e^{\lambda}+1}+c+\frac{1}{e^{\lambda}+1} \log \frac{1}{e^{\lambda}+1}+\frac{e^{\lambda}}{e^{\lambda}+1} \log \frac{e^{\lambda}}{e^{\lambda}+1} $$
by taking $c=\log \frac{e^{\lambda}+1}{e^{\lambda}}$, the minimal value equals zero and therefore $g(p)$ is an upper bound of $H(p)$.
\end{proof}
Choosing $\lambda=\log T$, we can upper bound the regret by
\begin{align}
    \label{eqn:dynamic_regret_bound}
\textbf{Regret}\le \frac{2(D_R^2+2n+\log T)}{\eta}+\frac{\eta T G_R^2}{2} +\frac{(n+\log T)\mathcal{P}}{\eta}+\frac{n\mathcal{P}+\log T \mathcal{P}+1}{2\eta}.
\end{align}




Choosing $\eta=\sqrt{\frac{\mathcal{P}(n+\log T)}{T L^2}}$, the regret is upper bounded by $O(L \sqrt{T(n+\log T) \mathcal{P}})$.

Now we have proven the following dynamic regret bound:
\begin{align*}
    \text{Dynamic Regret} =\sum_{t=1}^T \ell_t(x_t)-\min_{\varphi_t \in \mathcal{K}} \sum_{t=1}^T \ell_t(\varphi_t) =O(L \sqrt{T(n+\log T) \mathcal{P}}).
\end{align*}

Notice that 
\begin{equation*}
    \min_{\varphi_t \in \mathcal{K}} \sum_{t=1}^T \ell_t(\varphi_t)\le \min_{\varphi_t \in C_n} \sum_{t=1}^T \ell_t(\varphi_t)+O(L),
\end{equation*}
we conclude that
\begin{align*}
    \text{Dynamic Regret} =\sum_{t=1}^T \ell_t(x_t)-\min_{\varphi_t \in C_n} \sum_{t=1}^T \ell_t(\varphi_t) =O(L \sqrt{T(n+\log T) \mathcal{P}}).
\end{align*}

\end{proof}

For the rest of the section, we reason about choosing the optimal $\eta$, which depends on the actual path length that we do not know in advance. Algorithm \ref{algo:md+exp} tackles this problem by constructing a set $S$ of candidate learning rates, say $S=\{\frac{1}{4},\frac{1}{8},...\}$, then initiate $|S|$ number of Algorithm~\ref{algo:algomd} with different $\eta\in S$, and run weighted majority algorithm on top of these experts. 

The regret of such a two-layer algorithm can be decomposed as a sum of the regret of the base expert and the meta MW algorithm \cite{daniely2015strongly}. Since the path length $\mathcal{P}$ is upper bounded by $O(T)$, we need only $|S|=O(\log T)$ base experts, and the additional regret induced by the weighted majority algorithm is only $O(\sqrt{T \log \log T})$. This regret is dominated by the dynamic regret of the best base expert, and thus is negligible in the final bound. Meanwhile, the optimal $\eta$ is guaranteed to be either a small \rev{constant} which is already good, or lie between two candidate learning rates so that one of them is near-optimal.

More concretely, the best expert has regret upper bounded by $O(L \sqrt{T(n+\log T) \mathcal{P}})=\tilde{O}(L \sqrt{Tn \mathcal{P}})$ due to the existence of an optimal $\eta \in S$ by Lemma~\ref{thm:dynamic_regret}. The regret of the external MW algorithm in Algorithm \ref{algo:md+exp} is upper bounded by $O(\sqrt{T \log (|S|)})=\tilde{O}(\sqrt{T})$, a well-known regret bound of MW whose proof we omit here (see \citealp[Section 1.3]{hazan2016introduction} for example). Summing up, we get the final result $O(L\sqrt{T(n+\log(T))\mathcal{P}})$.


\section{Proof of Dynamic Regret Bounds Given Dynamical Models}\label{append:dynamical-model}

\subsection{Proof of Corollary~\ref{cor:dynamic-model-family}}
\begin{algorithm}
\caption{Dynamic Regret for Quantum Tomography Given Candidate Quantum Channels} \label{algo:md+exp_dynamical}
\begin{algorithmic}[1]
\STATE \textbf{Input:} a candidate set of quantum channels $\{\Phi^{(j)}\}$,$1\le j\le M$, a candidate set of $\eta$, $S=\{2^{-k-1}\mid 1\le k\le \log T\}$, constant $\alpha$.
\STATE Initialize $\log(T)M$ experts $E_{11},...,E_{\log(T)M}$, where $E_{kj}$ is an instance of Algorithm~\ref{algo:algomd_fixed_dynamical} with $\eta=2^{-k-1}$ and fixed quantum channel $\Phi^{(j)}$.
\STATE Set initial weights $w_1(k,j)=\frac{1}{\log(T)M}$.
\FOR{$t = 1, \ldots, T$}
\STATE Predict $x_t=\frac{\sum_{k,j} w_t(k,j) x_t(k,j)}{\sum_{k,j} w_t(k,j)}$, where $x_t(k,j)$ is the output of expert $E_{kj}$ at time $t$.
\STATE Observe loss function $\ell_t(\cdot)$.
\STATE Update the weights as
$$
w_{t+1}(k,j)=w_t(k,j) e^{-\alpha \ell_t(x_t(k,j))}.
$$
\STATE Send gradients $\nabla \ell_t(x_t(k,j))$ to each expert $E_{kj}$.
\ENDFOR
\end{algorithmic}
\end{algorithm}
\begin{algorithm}
\caption{DMD for Quantum Tomography Given Fixed Quantum Channel} \label{algo:algomd_fixed_dynamical}
\begin{algorithmic}[1]
\STATE \textbf{Input:}  domain $\mathcal{K}=(1-\frac{1}{T})C_n +\frac{1}{T 2^n}I$, step size $\eta <\frac{1}{2L}$, quantum channel $\Phi\colon C_n \mapsto C_n$.
\STATE Define $R(x)=\sum_{k=1}^{2^n} \lambda_k(x) \log \lambda_k(x)$, $\nabla R(x) := I + \log (x)$, and let $B_R$ denote the Bregman divergence defined by $R$.
\STATE Set $x_1=2^{-n} I$, and $y_1$ to satisfy $\nabla R(y_1)=\bf{0}$.
\FOR{$t = 1, \ldots, T$}
\STATE Predict $x_t$ and receive loss $\ell_t(\Tr(E_tx))$.
\STATE Define $\nabla_t= \ell'_t (\Tr(E_t x_t)) E_t$, where $\ell_t'(y)$ is a subderivative of $\ell_t$ with respect to $y$.
\STATE Update $y_{t+1}$ such that $\nabla R(y_{t+1})=\nabla R(x_t)-\eta \nabla_t$.
\STATE Update $ \hat{x}_{t+1} = \text{argmin}_{x \in \mathcal{K}} B_R(x||y_{t+1})$.
\STATE Update $x_{t+1} = \Phi(\hat{x}_{t+1})$.
\ENDFOR
\end{algorithmic}
\end{algorithm}
Similar to Appendix~\ref{append:proof-thm-dynamic}, we also need to mix $x$ with $\frac{I}{T2^n}:\tilde{x} = (1-\frac{1}{T})x + \frac{I}{T2^n},$ for both $x_t$ and $\varphi_t$. Notice that for any quantum channel $\Phi$, we have $\Phi(\tilde{x}) = (1-\frac{1}{T})\Phi(x)+\frac{I}{T2^n} = \widetilde{\Phi(x)}$, so we can use the same dynamical model in the new domain $\mathcal{K} = (1-\frac{1}{T})C_n + \frac{1}{T2^n}I$ and the path length of the mixed comparators $\{\widetilde{\varphi_t}\}$ may only decrease. Due to the $L$-Lipschitzness of loss functions, the regret is only hurt by $O(LT\frac{1}{T}) = O(L)$.

We first prove a lemma which gives a dynamic regret bound of Algorithm~\ref{algo:algomd_fixed_dynamical} when the learning rate $\eta$ is chosen optimally.
\begin{lemma}
\label{lem:dynamic_regret_fixed_dynamics}
Under the assumptions that the path length $\mathcal{P}=\sum_{t=1}^{T-1} \|\Phi(\varphi_t)-\varphi_{t+1}\|_\rev{1}\ge 1$ where $\|\cdot\|_\rev{1}$ is the nuclear norm and $\Phi:C_n\mapsto C_n$ is a quantum channel, the dynamic regret of Algorithm~\ref{algo:algomd_fixed_dynamical} is upper bounded by $O(L \sqrt{T(n+\log T) \mathcal{P}})$, if we choose $\eta=\sqrt{\frac{\mathcal{P}(n+\log T)}{T L^2}}$.
\end{lemma}
\begin{proof}

In general, quantum channels can be seen as  completely positive trace-preserving (CPTP) maps on the space of density operators. Ref.~\citet{lindblad1975completely} gave the following result:
\begin{lemma}[Contraction of CPTP maps with respective to relative entropy {\citealp{lindblad1975completely}}]
\label{lem:contraction_rel_entropy}
Let $T(\mathcal{H})$ be the operators in a Hilbert space $\mathcal{H}$ with finite trace and $T_{+}(\mathcal{H})$ be the positive elements in $T(\mathcal{H})$. If $\Phi$ is a completely positive trace-preserving map of $T(\mathcal{H})$ into itself, then for all $A, B \in T_{+}(\mathcal{H})$
$$
S(\Phi(A) \| \Phi(B)) \le S(A \| B),
$$
where $S(A\| B)$ is the relative entropy between $A$ and $B$.
\end{lemma}

Since density operators on finite Hilbert space is positive and has finite trace, and $B_R(x\|y) = S(x\| y)$ when $R$ is the von Neumann entropy, we can \rev{derive} that if $\Phi(\cdot)\colon C_n\to C_n$ is a quantum channel on $C_n$,
\begin{equation}
    \label{eqn:contraction_mapping}
    \forall x,y\in C_n,B_R(\Phi(x)\|\Phi(y)) \le B_R(x\|y).
\end{equation}

Writing $\ell_t(x) = \ell_t(\Tr(E_tx))$, similar to Eq.~\eqn{loss_ineq}, we can prove that
\begin{align*}
&\ell_t(x_t)-\ell_t(\varphi_t) \\
\le & \frac{1}{\eta}(B_R(\varphi_t\|x_t)- B_R(\varphi_t\|y_{t+1})+B_R(x_t\|y_{t+1})) \\ 
\le & \frac{1}{\eta}(B_R(\varphi_t\|x_t)- B_R(\varphi_t\|\hat{x}_{t+1})+B_R(x_t\|y_{t+1}))\\
= & \frac{1}{\eta}(B_R(\varphi_t\|x_t)-B_R(\varphi_{t+1}\|x_{t+1})+B_R(\varphi_{t+1}\|x_{t+1})-B_R(\varphi_t\|\hat{x}_{t+1})+B_R(x_t\|y_{t+1}))\\
\le & \frac{1}{\eta}(B_R(\varphi_t\|x_t)-B_R(\varphi_{t+1}\|x_{t+1})+B_R(\varphi_{t+1}\|x_{t+1})-B_R(\Phi(\varphi_t)\|x_{t+1})+B_R(x_t\|y_{t+1})),
\end{align*}
where the fourth inequality comes from Eq.~\eqn{contraction_mapping}.

By telescoping sums, we get
\begin{align*}
    & \textbf{Regret} \\ 
    \le & \frac{1}{\eta}(B_R(\varphi_1\|x_1)-B_R(\varphi_T\|x_T)) \\
    & + \frac{1}{\eta} \sum_{t=1}^{T-1} |B_R(\Phi(\varphi_t)\|x_{t+1})-B_R(\varphi_{t+1}\|x_{t+1})| 
    +\frac{1}{\eta} \sum_{t=1}^{T-1} B_R(x_t\|y_{t+1})\\
    \le & \frac{1}{\eta}B_R(\varphi_1\|x_1)+ \frac{\eta T G_R^2}{2} + \frac{1}{\eta} \sum_{t=1}^{T-1} |B_R(\Phi(\varphi_t)\|x_{t+1})-B_R(\varphi_{t+1}\|x_{t+1})|.
\end{align*}

Notice that this inequality is the same as Eq.~\eqn{bregman_distance_ineq} replacing $\varphi_t$ by $\Phi(\varphi_t)$, so we can follow the same proof of Eq.~\eqn{dynamic_regret_bound} and obtain a variant regret bound
\begin{align*}
    \textbf{Regret}\le \frac{D_R^2}{\eta}+\frac{\eta T G_R^2}{2}+ \frac{(n+\log T)\mathcal{P}}{\eta}+\frac{n\mathcal{P}+\log T \mathcal{P}+1}{2\eta},
\end{align*}
where $\mathcal{P} = \sum_{t=1}^{T-1}\|\varphi_{t+1}-\Phi(\varphi_t)\|_\rev{1}$.

Choosing $\eta=\sqrt{\frac{\mathcal{P}(n+\log T)}{T L^2}}$, the regret is upper bounded by $O(L \sqrt{T(n+\log T) \mathcal{P}})$.
\end{proof}

\begin{proof}[Proof of Corollary~\ref{cor:dynamic-model-family}]

We will prove that Algorithm~\ref{algo:md+exp_dynamical} can achieve the dynamic regret bound in this corollary and in the following proof we will use the notation in the description of Algorithm~\ref{algo:md+exp_dynamical}.

In Algorithm~\ref{algo:md+exp_dynamical}, we run weighted majority algorithm on $M\log(T)$ experts, each of which runs Algorithm~\ref{algo:algomd_fixed_dynamical} with quantum channel $\Phi^{(j)}$ and step size in $\{\frac{1}{4},\frac{1}{8},\ldots\}$.

From \lem{dynamic_regret_fixed_dynamics}, we can \rev{conclude} that the minimum dynamic regret bound of all $M\log(T)$ experts is
\begin{align*}
    O(L\sqrt{T(n+\log(T))\mathcal{P}'}),
\end{align*}
where $\mathcal{P}' = \min_{\Phi^{(j)}\in\{\Phi^{(1)},\ldots,\Phi^{(M)}\}}\sum_{t = 1}^{T-1} \|\varphi_{t+1}-\Phi^{(j)}(\varphi_{t})\|_\rev{1}$.

\rev{In Lemma 1 of the arXiv version of ~\citet{zhang2018adaptive}}, they give a regret bound of the weighted majority algorithm:
\begin{align*}
\sum_{t =1}^{T}\ell_t(x_t) \le \sum_{t =1}^{T}\ell_t(x_t(k,j)) + \frac{c\sqrt{2T\ln(\frac{1}{w_1(k,j)})}}{4},
\end{align*}
for any $1\le k\le\log(T),1\le j\le M$ choosing $\alpha = \sqrt{\frac{8}{Tc^2\ln(\frac{1}{w_1(k,j)})}}$.

Let $k,j$ be the optimal one of all $M\log(T)$ experts, we have
\begin{align*}
\sum_{t =1}^{T}\ell_t(x_t) \le& \sum_{t =1}^{T}\ell_t(\varphi_t) + c_1L\sqrt{T(n+\log(T))\mathcal{P}'} + \frac{c\sqrt{T\ln(\log(T)M)}}{4}\\
=& \sum_{t =1}^{T}\ell_t(\varphi_t) + O(L\sqrt{T(n+\log(T))\mathcal{P}'} +\sqrt{T\log(M\log(T))}),
\end{align*}
which concludes our result.
\end{proof}

\subsection{Proof of Corollary~\ref{cor:dynamic-model-local-channel}}

In general, the possible dynamical models can be a continuous set, and we can take it to be the set of all local quantum channels for example. A quantum channel on $n$ qubits is said to be $l$-local if it can be written as a product channel $\Phi = \Phi_1\otimes I$, where $\Phi_1$ only influences $l$ qubits.

\begin{lemma}[Stinespring's dilation theorem \citealp{Stinespring1955PositiveFO}]
\label{lem:stinespring}
Let $\Phi: C_n \rightarrow C_n$ be a completely positive and trace-preserving map on the space of density operators of a Hilbert space $\mathcal{H}_s$. Then there exists a Hilbert space $\mathcal{H}_e$ with $\dim{\mathcal{H}_e}= (\dim{\mathcal{H}_s})^2$and a unitary operation $U$ on $\mathcal{H}_s \otimes \mathcal{H}_e$ such that
$$
\Phi(\rho)=\Tr_{e}( U(\rho \otimes|0\rangle\langle 0|_e) U^{\dagger}
)$$
for all $\rho \in C_n$, where $\Tr_{e}$ denotes the partial trace on $\mathcal{H}_e$.
\end{lemma}

By Stinespring's dilation theorem, we can represent the quantum channel on a $2^l$ dimension Hilbert space $\mathcal{H}_s$ by partial trace of a pure quantum channel on a $2^{3l}$ dimension Hilbert space $\mathcal{H}_s\otimes \mathcal{H}_e$, so we only need to find $\epsilon$-net of unitary operators.

\begin{lemma}[$\epsilon$-net of $U_l$ {\citealp[Lemma 4.4]{knill1995approximation}}]
\label{lem:eps-net}
Let $U_l$ be the set of unitaries on Hilbert space $\mathcal{H}$ with $\dim{\mathcal{H}} = 2^l$. For $\delta>0$, there is a subset $U_{l, \delta}$ of $U_{l}$ with
$$
\left|U_{l, \delta}\right|\leq\left(\frac{2}{\delta}\right)^{2^{4 l}}
$$
such that for any $U \in U_{l}$, there exists $V \in U_{l, \delta}$ with $\|U-V\|_{2} \leq \delta$.
\end{lemma}

From \lem{stinespring} and \lem{eps-net}, we can construct an $\epsilon$-net of general quantum channels:
\begin{lemma}[$\epsilon$-net of quantum channels]
\label{lem:eps-net_channel}
Let $S_l$ be the set of all quantum channels on density operators on Hilbert space $\mathcal{H}$ with $\dim{\mathcal{H}} = 2^l$. For any $1> \delta > 0$, there is a subset $S_{l,\delta}$ of $S_l$ with $$
\left|S_{l, \delta}\right| \leq\left(\frac{2}{\delta}\right)^{2^{12 l}},
$$
such that for any $\Phi\in S_l$, there exists $\tilde{\Phi}\in S_{l,\delta}$ such that $\sup_{x\in C_l}\|\Phi(x)-\tilde{\Phi}(x)\|_\rev{1}\le 3\delta$.
\end{lemma}
\begin{proof}
Let $\mathcal{H}_e$ be a Hilbert space with $\dim(\mathcal{H}_e) = 2^{2l}$ and $\mathcal{H}_a = \mathcal{H}\otimes \mathcal{H}_e$. Let $U_{3l,\delta}$ be the $\epsilon$-net of unitaries on $\mathcal{H}_a$ in \lem{eps-net}.

Then, we can construct $S_{l,\delta} = \{\Phi \mid \exists U\in U_{3l,\delta}, \forall x\in C_l, \Phi(x)=\Tr_{e}( U(x \otimes|0\rangle\langle 0|_e) U^{\dagger}\}$.

For an arbitrary $\Phi'\in S_l$, from \lem{stinespring}, we can \rev{show} that there exists a Hilbert space $\mathcal{H}_c$ with $\dim{\mathcal{H}_c} = 2^{2l}$ and a unitary operation $U'$ on $\mathcal{H}\otimes \mathcal{H}_c$ such that
$$
\Phi(x)=\Tr_{c}( U'(x \otimes|0\rangle\langle 0|_c) U'^{\dagger})
$$
for any $x \in C_l$.

Notice that $\mathcal{H}\otimes\mathcal{H}_e$ and $\mathcal{H}\otimes\mathcal{H}_c$ are isomorphic since they are Hilbert space with the same dimension, so there exists $V'\in U_{3l,\delta}$ such that $\|U'-V'\|_2\le \delta$.

From the definition of $S_{l,\delta}$, \rev{there exists} $\tilde{\Phi}\in S_{l,\delta}$ such that
$$
\tilde{\Phi}(x)=\Tr_{e}( V'(x \otimes|0\rangle\langle 0|_e) V'^{\dagger})
$$
for any $x\in C_l$.

Then, we can \rev{show} that for any $x\in C_l$
\begin{align*}
    &\|\Phi(x)-\tilde{\Phi}(x)\|_\rev{1} \\
    = &
    \|\Tr_{e}( U'(x \otimes|0\rangle\langle 0|_e) U'^{\dagger}- V'(x \otimes|0\rangle\langle 0|_e) V'^{\dagger})\|_\rev{1} \\
    \le & \|U'(x \otimes|0\rangle\langle 0|_e) U'^{\dagger}- V'(x \otimes|0\rangle\langle 0|_e) V'^{\dagger}\|_\rev{1}\\
    \le & 2\|(U'-V')(x \otimes|0\rangle\langle 0|_e)V'^{\dagger}\|_\rev{1} + \|(U'-V')(x \otimes|0\rangle\langle 0|_e)(U'-V')^{\dagger}\|_\rev{1}\\
    \le & 2\|U'-V'\|_2\|(x \otimes|0\rangle\langle 0|_e)V'^{\dagger}\|_\rev{1} + \|U'-V'\|_2^2\|x \otimes|0\rangle\langle 0|_e\|_\rev{1}
    \le 2\delta + \delta^2 \le 3\delta,
\end{align*}
where in the first equality, we let $\Phi$ be a quantum channel on the density operators on $\mathcal{H}\otimes\mathcal{H}_e$ since $\mathcal{H}_s$ and $\mathcal{H}_e$ are isomorphic.

The size of $S_{l,\delta}$ is the same as $|U_{3l,\delta}|$ which is $\left(\frac{2}{\delta}\right)^{2^{12l}}$.
\end{proof}

With \lem{eps-net_channel}, we can construct an $\epsilon$-net of $l$-local quantum channels.

\begin{lemma}[$\epsilon$-net of $l$-local quantum channels]
\label{lem:eps_net_local_channels}
Let $S_{n,l}$ be the set of all $l$-local quantum channels on $n$ qubits. For any $1> \delta > 0$, there is a subset $S_{n,l,\delta}$ of $S_{n,l}$ with
$$
|S_{n,l,\delta}|\le \binom{n}{l}\left(\frac{6}{\delta}\right)^{2^{12l}},
$$
such that for any $\Phi\in S_{n,l}$, there exists $\tilde{\Phi}\in S_{n,l,\delta} $ such that $\sup_{x\in C_n}\|\Phi(x)-\tilde{\Phi}(x)\|_\rev{1} \le \delta$.
\end{lemma}
\begin{proof}
From \lem{eps-net_channel}, we can construct a collection of sets $\{W_{i_1,\ldots,i_l}\mid 1\le i_1 <\cdots < i_l \le n\}$. Each $W_{i_1,\ldots,i_l}$ is a set of quantum channels on $C_n$ and contains extension of channels in $S_{l,\frac{\delta}{3}}$ in \lem{eps-net_channel} acting on qubits $i_1,\ldots,i_l$ and all operators in it perform identity on the other $n-l$ qubits, so $|W_{i_1,\ldots,i_l}| = \left(\frac{6}{\delta}\right)^{2^{12l}}$.

Let $S_{n,l,\delta}$ be the union of all $W$ in $\{W_{i_1,\ldots,i_l}\mid 1\le i_1 <\cdots < i_l \le n\}$, then we can \rev{conclude} for any quantum channel $\Phi \in S_{n,l}$, there exists $\tilde{\Phi}(x) \in S_{n,l,\delta}$ with $\sup_{x\in C_n}\|\Phi(x)-\tilde{\Phi}(x)\|_\rev{1}\le \delta$ and $|S_{n,l,\delta}| \le \sum_{i_1,\ldots,i_l}|W_{i_1,\ldots,i_l}| \le \binom{n}{l}\left(\frac{6}{\delta}\right)^{2^{12l}}$.
\end{proof}

\begin{proof}[Proof of Corollary~\ref{cor:dynamic-model-local-channel}]
Let
\begin{align}
    \label{eq:l_local_path_length}
    \mathcal{P}' = \min_{\Phi\in S_{n,l}}\sum_{t = 1}^{T-1} \|x_{t+1}-\Phi(x_{t})\|_\rev{1},
\end{align}
where $S_{n,l}$ is the set of all $l$-local quantum channels on $n$ qubits.

For any $\Phi\in S_{n,l}$, there exists $\tilde{\Phi}$ such that $\sup_{x\in C_n}\|\Phi(x)-\tilde{\Phi}(x)\|_\rev{1} \le \delta$, then we have
\begin{align*}
    \|x_{t+1}-\Phi(x_t)\|_\rev{1} &\ge \|x_{t+1}-\tilde{\Phi}(x_t)\|_\rev{1} - \|\Phi(x)-\tilde{\Phi}(x)\|_\rev{1} \ge  \|x_{t+1}-\tilde{\Phi}(x_t)\|_\rev{1} - \delta,
\end{align*}
for any $1\le t\le T$.
Thus, we have
\begin{align*}
     &\min_{\Phi\in S_{n,l}}\sum_{t = 1}^{T-1} \|x_{t+1}-\Phi(x_{t})\|_\rev{1} \ge  \min_{\Phi\in S_{n,l,\delta}}\sum_{t = 1}^{T-1} \|x_{t+1}-\Phi(x_{t})\|_\rev{1} - T\delta.
\end{align*}

From Corollary~\ref{cor:dynamic-model-family}, we can \rev{show} that the regret of Algorithm~\ref{algo:md+exp_dynamical} with the candidate set of quantum channels to be the $S_{n,l,\delta}$ in \lem{eps_net_local_channels} satisfies
\begin{small}
    \begin{align*}
    &\textbf{Regret} \\
    \le & c_1L\sqrt{T(n+\log(T))\min_{\Phi\in S_{n,l,\delta}}\sum_{t = 1}^{T-1} \|x_{t+1}-\Phi(x_{t})\|_\rev{1}} + c_2\sqrt{T\log\log(T)+T\log\biggl(\binom{n}{l}\left(\frac{6}{\delta}\right)^{2^{12l}}\biggr)}\\
    \le & c_1L\sqrt{T(n+\log(T))(\min_{\Phi\in S_{n,l}}\sum_{t = 1}^{T-1} \|x_{t+1}-\Phi(x_{t})\|_\rev{1} + T\delta)}+ c_2\sqrt{T\log\log(T)+T\log\biggl(\binom{n}{l}\left(\frac{6}{\delta}\right)^{2^{12l}}\biggr)}\\
    = & c_1L\sqrt{T(n+\log(T))(\mathcal{P}'+3T\delta)} +c_2\sqrt{T\log\log(T) + T\left(2^{12l}\log(\frac{6}{\delta})+l\log(n)\right)}.
\end{align*}
\end{small}

Let $\delta = \frac{1}{T}$, we can achieve a regret of $\tilde{O}(L\sqrt{nT\mathcal{P}'}+2^{6l}\sqrt{T})$ when $\mathcal{P}' \ge 1$.
\end{proof} 


\section{Other Proofs}\label{append:small-proof}
\subsection{$K$-outcome Measurement Cases}\label{append:k-outcome}
We can generalize the two-outcome measurements in our results to $K$-outcome measurements. The only issue of replacing two-outcome measurements with $K$-outcome measurements is that the output of the $K$-outcome measurement is a probability distribution rather than a scalar, so we need a new loss function to measure the regret. Let $\mathbf{p}_t = (\mathrm{Tr}(E_{t,1}x_t),\ldots,\mathrm{Tr}(E_{t,K}x_t))$ and $\mathbf{b}_t = (b_{t,i})_{i=1}^K$ be the probability distribution corresponding to the output of the measurement actting on the ground truth state $\rho_t$. In this case, we can define the loss function at time $t$ to be 
\begin{equation}
    \label{eqn:loss_ineq_K}
    \ell_t(x_t) = d_{\mathrm{TV}}(\mathbf{p_t},\mathbf{b_t}) = \frac{1}{2}\sum_{i=1}^K\left|\operatorname{Tr}\left(E_{t,i} x_t\right)-b_{t,i}\right|.
\end{equation}
This is a convex function so the proof in previous sections still holds and we only need to determine $G_R$ which is the upper bound on $\|\nabla_t\|$. 

First, we prove the following inequality on the spectral norm of linear combination of $E_i$. For any $K$-outcome measurement $\mathcal{M}$ described by $\{E_1,\ldots,E_K\}$ and $K$ real numbers $a_1,\ldots,a_K$ such that $|a_i|\le L$ for all $i\in[K]$, we have
\begin{align}
    \|\sum_{i=1}^K a_iE_i\|  &= \max_{v\in \mathbb{C}^n,\|v\|_2 = 1}v^{\dagger}(\sum_{i=1}^K a_iE_i)v 
    \le \max_{v\in \mathbb{C}^n,\|v\|_2 = 1}\sum_{i=1}^K|a_i|v^{\dagger}E_i v \notag \\
    &\le \max_{v\in \mathbb{C}^n,\|v\|_2 = 1} Lv^{\dagger}(\sum_{i=1}^K E_i)v = L,
\end{align}

where we use spectral norm in the first line and the first inequlity comes from $0\preceq E_i$. Then we have 
\begin{equation}
    \|\nabla_t\| = \|\sum_{i=1}^K \frac{1}{2}\mathrm{sgn}(\mathrm{Tr}(E_{t,i} x_t)-b_{t,i})E_{t,i}\| \le \frac{1}{2},
\end{equation}
where $\mathrm{sgn}(x)$ is the sign function. 

Therefore, $G_R = \frac{1}{2}$ is a valid parameter in this case. Since we have $G_R=L$ in previous two-outcome cases, we can simply repalce all $L$ with $\frac{1}{2}$ in our results in two-outcome cases to obatin an algorithm in $K$-outcome cases with loss function in Eq.~(\ref{eqn:loss_ineq_K}).

\subsection{Proof of Corollary {\ref{cor:path-length-mistake}}}
\begin{proof}
To achieve the mistake bound, we run $\mathcal{A}$ in a subset of the iterations. Specifically, we only update $x_t$ if $\ell_t(\Tr(E_t x_t)) > \frac{2}{3}\epsilon$. Since $\ell_t(\Tr(E_t x_t)) > \frac{2}{3}\epsilon$ whenever $\left|\Tr\left(E_{t}  x_{t}\right)-\Tr\left(E_{t} \rho_t\right)\right|>\epsilon$, the number of the mistakes $\mathcal{A}$ makes is at most the number of the updates $\mathcal{A}$ makes.

Let the sequence of $t$ when $\mathcal{A}$ updates $x_t$ be $\{t_i\}_{i=1}^{T}$. Since $\{\rho_{t_i}\}$ is a valid comparator sequence (sub-sequence of a $k$-shift sequence shifts at most $k$ times) with $\sum_{i=1}^{T}\ell_{t_i}(\Tr(E_{t_i}\rho_{t_i})) \leq \frac{1}{3}\epsilon T$\rev{, the} number of updates $T$ satisfies the inequality $\frac{2}{3}\epsilon T\le \frac{1}{3}\epsilon T + \tilde{O}(L \sqrt{Tn \mathcal{P}})$, \rev{where $\tilde{O}(L \sqrt{Tn \mathcal{P}})$ comes from the fact that the algorithm's regret compared to the best expert is upper bounded by $\tilde{O}(L \sqrt{Tn \mathcal{P}})$}. Thus, it implies $T = \tilde{O}(\frac{L^2 n \mathcal{P}}{\epsilon^2})$, which completes the proof.
\end{proof}

\subsection{Proof of \rev{Lemma \ref{lem:SA-Regret}}}
\begin{proof}
The CBCE algorithm in \citet{jun2017improved} is a meta-algorithm which uses any online convex optimization algorithm $\mathcal{B}$ as a black-box algorithm and obtain a strongly adaptive regret with an $O(\sqrt{\tau\log (T)})$ overhead for any interval length $\tau$. Its guarantee is formalized in the following lemma.

\begin{lemma}[SA-Regret of $\text{CBCE}\langle\mathcal{B}\rangle$, {\citealp[Theorem 2]{jun2017improved}}]\label{lem:CBCE}
Assume the loss function $\ell_{t}(x)$, $x\in \mathcal{K}$ is convex and maps to $[0,1], \forall t \in \mathbb{N}$ and that the black-box algorithm $\mathcal{B}$ has static regret $R_{T}^{\mathcal{B}}$ bounded by $c T^{\alpha}$, where $\alpha \in(0,1)$. Let $I=\left[I_{1}, I_{2}\right]$. The SA-Regret of \text{CBCE} with black-box $\mathcal{B}$ on the interval I is bounded as follows:
$$
\begin{aligned}
R^{\mathrm{CBCE}\langle\mathcal{B}\rangle}(I) &=
\min_{\varphi\in \mathcal{K}}\sum_{t \in I}\left(\ell_{t}\left(\mathbf{x}_{t}^{\mathcal{M}\langle\mathcal{B}\rangle}\right)-\ell_{t}(\varphi)\right) \leq \frac{4}{2^{\alpha}-1} c|I|^{\alpha}+8 \sqrt{|I|\left(7 \log \left(I_{2}\right)+5\right)} \\
&=O\left(c|I|^{\alpha}+\sqrt{|I| \log (I_{2})}\right).
\end{aligned}
$$
\end{lemma}

where $\mathcal{M}\langle\mathcal{B}\rangle$ denotes the meta-algorithm. Note that this bound can be obtained with any bounded convex function sequence $\ell_t$.

For the choice of black-box algorithm, the RFTL based algorithm we used as the black-box function in \citet{aaronson2018online} achieves an $O(\sqrt{Tn})$ bound on the static regret when $\ell_t$ are $\ell_1$ or $\ell_2$ loss.

\begin{lemma}[{\citealp[Theorem 2]{aaronson2018online}}]\label{lem:OnlineQ-Thm2}
Let $E_{1}, E_{2}, \ldots$ be a sequence of two-outcome measurements on an $n$-qubit state presented to the learner, and suppose $\ell_{t}$ is convex and L-Lipschitz. Then there is an explicit learning strategy that guarantees regret $R_{T}=O(L \sqrt{T n})$ for all $T$. Specifically, when applied to $\ell_{1}$ loss and $\ell_{2}$ loss, the RFTL algorithm achieves regret $O(\sqrt{T n})$ for both.
\end{lemma}

Note that both $\ell_1$ and $\ell_2$ norm versions of our loss function $\ell_t$ are $2$-Lipschitz and bounded. Therefore, we can directly obtain an algorithm with strongly adaptive regret bound by combining Lemma \ref{lem:CBCE} and Lemma \ref{lem:OnlineQ-Thm2} with $\alpha=1/2$.
\end{proof}

\subsection{Proof of \rev{Theorem \ref{thm:k_shift regret}}}
\begin{proof}
Given any $1=t_1\le t_2\le \cdots \le t_{k+1}=T+1$ and $\varphi_1,\varphi_2,\ldots, \varphi_k\in C_n$, $\mathcal{A}$ guarantees that
$$
\begin{aligned}
    &\sum_{t= 1}^T \ell_t(E_t x_t^{\mathcal{A}})-\sum_{j = 1}^{k}\sum_{t = t_j}^{t_{j+1}-1}\ell_t(E_t\varphi_j) \\
    \le & \sum_{j=1}^k cL\sqrt{n(t_{j+1}-t_j)\log(T)} \\
    \le & cL\sqrt{knT\log(T)},
\end{aligned}
$$
where $c$ is some constant. In the second line we use the result of \rev{Lemma \ref{lem:SA-Regret}}, and in the third line we use the Cauchy–Schwarz inequality.
\end{proof}

\subsection{Proof of Corollary \ref{cor:k-shift-mistake}}

\begin{proof}
To achieve the mistake bound, we run $\mathcal{A}$ in a subset of the iterations. Specifically, we only update $x_t$ if $\ell_t(\Tr(E_t x_t)) > \frac{2}{3}\epsilon$. Since $\ell_t(\Tr(E_t x_t)) > \frac{2}{3}\epsilon$ whenever $\left|\Tr\left(E_{t}  x_{t}\right)-\Tr\left(E_{t} \rho_t\right)\right|>\epsilon$, the number of the mistakes $\mathcal{A}$ makes is at most the number of the updates $\mathcal{A}$ makes.

Let the sequence of $t$ when $\mathcal{A}$ updates $x_t$ be $\{t_i\}_{i=1}^{T}$. Since $\{\rho_{t_i}\}$ is a valid comparator sequence (sub-sequence of a $k$-shift sequence shifts at most $k$ times) with $\sum_{i=1}^{T}\ell_{t_i}(\Tr(E_{t_i}\rho_{t_i})) \leq \frac{1}{3}\epsilon T$, from \rev{Theorem~\ref{thm:k_shift regret}}, we can \rev{conclude} that the number of updates $T$ satisfies the inequality $\frac{2}{3}\epsilon T\le \frac{1}{3}\epsilon T + c\sqrt{knT\log(T)}$. Thus, it implies $T = O(\frac{kn}{\epsilon^2}\log^2(\frac{kn}{\epsilon^2}))$, which completes the proof.
\end{proof}